\documentclass[letter,11pt]{article}
\usepackage[margin=1in]{geometry}
\usepackage{amsthm}
\usepackage{amsfonts}
\usepackage{amsmath,graphicx,multicol}
\usepackage{amssymb}
\usepackage[utf8]{inputenc}
\usepackage[english]{babel}
\usepackage{bm}
\usepackage{epstopdf}
\usepackage{hyperref}
\usepackage{enumitem}
\usepackage{algorithm,algorithmic}
\usepackage{cleveref}
\usepackage{xcolor}
\usepackage{subcaption}
\usepackage{url}
\newcommand{\E}{\mathbb{E}}

\usepackage{cite}
\usepackage{booktabs} 
\usepackage{mathtools}

\newtheorem{assumption}[]{Assumption}
\newtheorem{remark}{\textbf{Remark}}
\newtheorem{proposition}{Proposition}[]
\newtheorem{theorem}{Theorem}[]
\newtheorem{corollary}{Corollary}[theorem]
\newtheorem{lemma}[]{Lemma}
\usepackage{multirow}
\usepackage{makecell}   
\usepackage{colortbl}   
\usepackage{hhline}
\definecolor{tabcolor}{RGB}{245,237,227} 

\usepackage{titletoc}
\usepackage{enumitem}
\definecolor{tabcolor}{RGB}{245,237,227} 
\hypersetup{bookmarksdepth=2} 

\title{\vspace{-10mm}\textbf{RAMPAGE: RAndomized Mid-Point for debiAsed Gradient Extrapolation}}
\date{}
\vspace{-0.05in}
\author{Zhankun Luo$^\dagger$, M. Berk Sahin$^\dagger$, Antesh Upadhyay$^\dagger$, Behzad Sharif, Abolfazl Hashemi\thanks{Authors are with the School of Electrical and Computer Engineering, Purdue University, West Lafayette, IN 47907, USA. $^\dagger$ denotes equal contribution.}}
\begin{document}
\maketitle

\vspace{-5mm}
\begin{abstract}
\vspace{-7mm}
\noindent{
}
\\\\\noindent
A celebrated method for Variational Inequalities (VIs) is Extragradient (EG), which can be viewed as a standard discrete-time integration scheme. With this view in mind, in this paper we show that EG may suffer from discretization bias when applied to non-linear vector fields, conservative or otherwise. To resolve this discretization shortcoming, we introduce RAndomized Mid-Point for debiAsed Gradient Extrapolation (RAMPAGE) and its variance-reduced counterpart, RAMPAGE+, which leverages antithetic sampling. In contrast with EG, both methods are unbiased. Furthermore, leveraging negative correlation, RAMPAGE+ acts as an unbiased, geometric path-integrator that completely removes internal first-order terms from the variance, provably improving upon RAMPAGE. We further demonstrate that both methods enjoy provable $\mathcal{O}(1/k)$ convergence guarantees for a range of problems including root finding under co-coercive, co-hypomonotone, and generalized Lipschitzness regimes. Furthermore, we introduce symmetrically scaled variants to extend our results to constrained VIs. Finally, we provide convergence guarantees of both methods for stochastic and deterministic smooth convex-concave games. Somewhat interestingly, despite being a randomized method, RAMPAGE+ attains purely deterministic bounds for a number of the studied settings.
\end{abstract}
\section{Introduction}
In recent years, modern learning and decision-making paradigms, e.g., Generative Adversarial Networks (GANs), adversarial training, Distributionally Robust Optimization (DRO), robust reinforcement learning, and multi-agent competitive systems, have triggered a renewed interest in the study of the broader domain of min-max optimization and Variational Inequalities (VIs).
Formally, consider the problem of finding an equilibrium state $\theta^* \in \mathcal{X} \subset \mathbb{R}^p$ such that the variational condition $\langle F(\theta^*), \theta - \theta^* \rangle \ge 0,$ for all $\theta \in \mathcal{X},$ is satisfied,
where $F: \mathbb{R}^p \to \mathbb{R}^p$ represents the driving operator of the system. In the unconstrained regime where $\mathcal{X} = \mathbb{R}^p$, this condition reduces to the root-finding problem $F(\theta^*) = 0.$
This formalism encompasses many applications:
In pure minimization tasks, the operator represents a conservative gradient field, $F(\theta) = \nabla f(\theta)$, where the underlying geometry is governed by potential energy dissipation. However, in min-max games and generalized VIs, the operator $F$ frequently exhibits non-conservative, skew-symmetric (rotational) components (e.g., the classical bilinear min-max games). 

\subsection{Motivation}\label{sec:motivation}
The discrete update of most first-order algorithms is fundamentally an attempt to numerically integrate the continuous flow defined by the ordinary differential equation
\begin{equation} \label{eq:continuous_flow}
\dot{\theta}(\tau) = -F(\theta(\tau)) \qquad \equiv \qquad \theta_{t+1} = \theta_t - \int_{t}^{t+1} F(\theta(s)) ds.
\end{equation}
With this view, a fundamental observation of modern optimization is that naive discrete-time integration schemes, such as the standard Forward Euler method (corresponding to Gradient Descent for minimization problems and Gradient Descent-Ascent for min-max problems), exhibit divergence when subjected to vector fields with rotational components. 

To stabilize the discrete dynamics, the field overwhelmingly relies on Extragradient (EG), a celebrated method that constructs an extrapolated intermediate state to approximate an implicit backward step. For a step size $\eta > 0$, the update is defined by
\begin{equation} \label{eq:eg}
\tag{EG}
\theta_{t+1} = \theta_t - \eta F\big(\theta_t - \eta F(\theta_t)\big).
\end{equation}
While \eqref{eq:eg} provides stability in smooth, well-conditioned bilinear games and other settings, an analysis of its local truncation error reveals a bias that may hinder its efficacy and stability in highly non-linear, high-dimensional settings. Let us state a concrete motivating example next.

\noindent\textbf{Infinite-Width GANs as Motivating Example:}
As a motivating example which is highly non-linear, consider GANs in the Neural Tangent Kernel (NTK) regime \cite{jacot2018neural}. Consider a Dirac-GAN architecture \cite{mescheder2018training} where the true data distribution is defined as a singular point mass at the origin $\theta^* = 0$. The generator's parameter is constrained to the unit circle $\mathbb{S}^1$, parameterized by the angle $\theta \in [-\pi, \pi]$. For the discriminator $D$, we assume the infinite-width limit, leading to an isomorphism between the discriminator and a linear functional operating within an RKHS $\mathcal{H}$. This space is uniquely governed by the deterministic covariance kernel $k(\cdot, \cdot)$ induced by the selected activation function \cite{jacot2018neural}. To give a concrete example, when the discriminator utilizes the Rectified Linear Unit (ReLU) as activation, i.e., $\phi(z) = \max(0, z)$, the kernel evaluates to
$k(|\theta|) = \frac{1}{\pi} \big( \sin |\theta| + (\pi - |\theta|) \cos |\theta| \big)$ \cite{arora2019fine,li2019enhanced,golikov2022neural}.

To ensure a bounded supremum for the functional objective, we consider the following regularized min-max game
\begin{equation}
\min_{\theta \in [-\pi, \pi]} \max_{D \in \mathcal{H}} \left( D(0) - D(\theta) - \frac{1}{2}\|D\|_{\mathcal{H}}^2 \right).
\end{equation}
As derived in Appendix \ref{sec:gan}, solving the regularized inner maximization and exploiting the rotational invariance of the kernel on the periodic domain reduces the generator's trajectory to the gradient flow with operator $F(\theta) = -\nabla_\theta k(|\theta|).$
The even periodic kernel $k(|\theta|)$ admits a uniform Fourier cosine expansion parameterized by spectral coefficients $\lambda_n$ (see, e.g. Chapter 1.1 in \cite{katznelson2004introduction}); writing $k(|\theta|)=\lambda_0+2\sum_{n=1}^{\infty}\lambda_n\cos(n\theta)$, taking derivatives yields $F(\theta) = \sum_{n=1}^\infty 2n \lambda_n \sin(n \theta).$

The asymptotic decay rate of the sequence $\lambda_n$ depends on the smoothness of the induced angular kernel. The induced kernel coefficients satisfy $\lambda_n=\mathcal{O}(n^{-(2s+1)})$ for an activation function in the Sobolev space $W^{s,2}$ (see Chapter 1.4 in~\cite{katznelson2004introduction}), then the differential operator linearly scales each harmonic coefficient to yield a decay rate of $\mathcal{O}(n^{-2s})$.

For the explicit ReLU angular kernel above, the nonzero Fourier coefficients satisfy $\lambda_n=\mathcal{O}(n^{-4})$. The differential operator then linearly scales each harmonic coefficient, so the corresponding nonzero coefficients of the vector field $F(\theta)$ decay as $\mathcal{O}(n^{-3})$. Thus, even in this simple infinite-width Dirac-GAN model, the generator vector field contains an infinite sequence of polynomially decaying higher harmonics. These high-frequency components produce a nonlinear kernel-induced flow. Motivated by this, our numerical results presented in Appendix \ref{sec:exp} and summarized in Fig. \ref{fig:main} are devised to test the impact of such high-frequency perturbations on learning and optimization dynamics.
\subsection{Proposed Idea}
With the above motivating example, to understand the limitations of \eqref{eq:eg}, we evaluate its truncation error with respect to \emph{an ideal estimator.}

As is evident by the integral form of \eqref{eq:continuous_flow}, an ideal discrete stabilization vector should capture the \emph{mean field} acting upon the state as it traverses the local exploration radius. We formalize this by defining the exact line integral over a parameterized exploration segment $\gamma(s) = \theta_t - c\eta s F(\theta_t)$ for $s \in [0, 1]$ where $c \ge 1$ dictates the integration scale. We set $c = 1$ for conservative vector fields and $c>1$ for general non-conservative fields. In the present section and in our main algorithms, we adopt $c=2$ so that the midpoint $s=1/2$ coincides with the standard EG extrapolation point (see Appendix \ref{app:scale} for more details). We thus define the continuous integral of the operator over this segment as
\begin{equation} \label{eq:intro_line_integral}
\mathcal{I}_{\gamma} = \int_0^1 F\big(\theta_t - 2\eta s F(\theta_t)\big) ds,
\end{equation}
which acts as a low-pass filter over the vector field, given the smoothing properties of integration.

While $\mathcal{I}_{\gamma}$ provides superior continuous-time stabilization, it is infeasible to compute it in general and as such \eqref{eq:eg} evaluates $\mathcal{I}_{\gamma}$ at the single extrapolated midpoint $s=1/2$, generating the update field $F_{EG} = F\big(\theta_t - \eta F(\theta_t)\big)$. Expanding $F_{EG}$ via Taylor series and subtracting it from the exact integral $\mathcal{I}_{\gamma}$ leads to the leading-order bias as $\mathcal{I}_{\gamma} - F_{EG} = \frac{1}{6}\eta^2 \nabla^2 F(\theta_t)\big[F(\theta_t), F(\theta_t)\big] + \mathcal{O}(\eta^3)$ (see Appendix \ref{app:variance} for more details).
This shows \eqref{eq:eg} is a biased estimator of $\mathcal{I}_{\gamma}$ and the bias scales with the Hessian tensor $\nabla^2 F(\theta_t)$.

\begin{figure*}[t]
\centering

\begin{subfigure}[t]{0.32\textwidth}
 \centering
 \includegraphics[width=\textwidth]{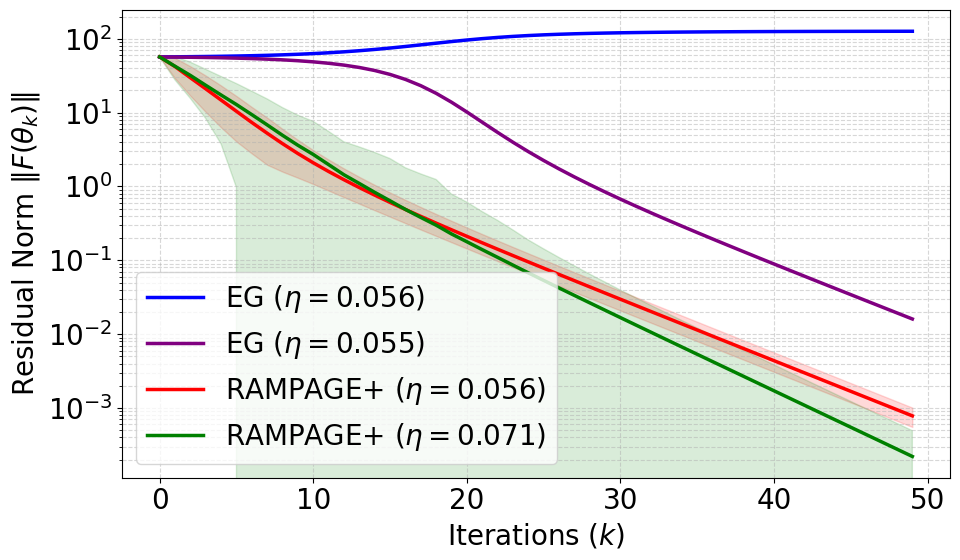}
 \caption{\footnotesize Conservative Polynomial Field}
\end{subfigure}\hfill
\begin{subfigure}[t]{0.32\textwidth}
 \centering
 \includegraphics[width=\textwidth]{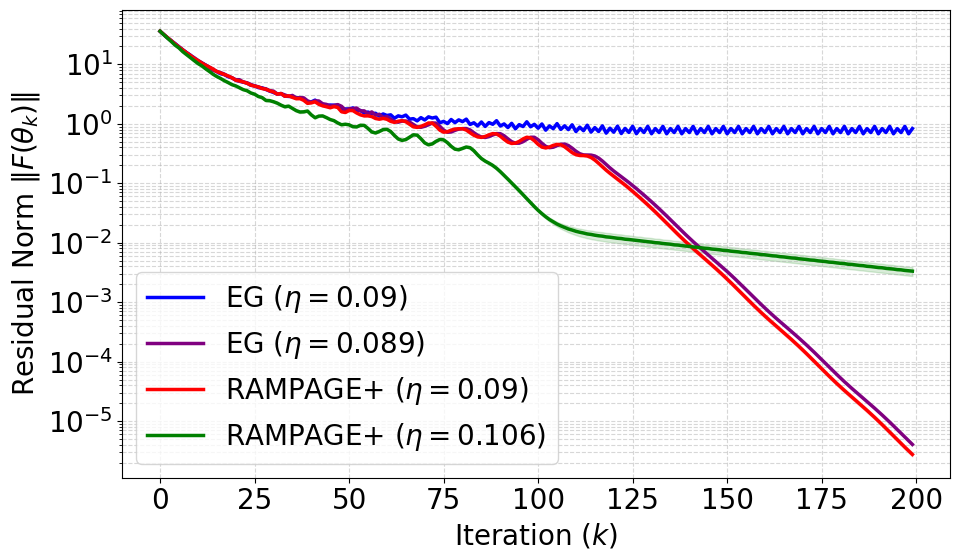}
 \caption{\footnotesize High-Frequency Min-Max Game}
\end{subfigure}\hfill
\begin{subfigure}[t]{0.32\textwidth}
 \centering
 \includegraphics[width=\textwidth]{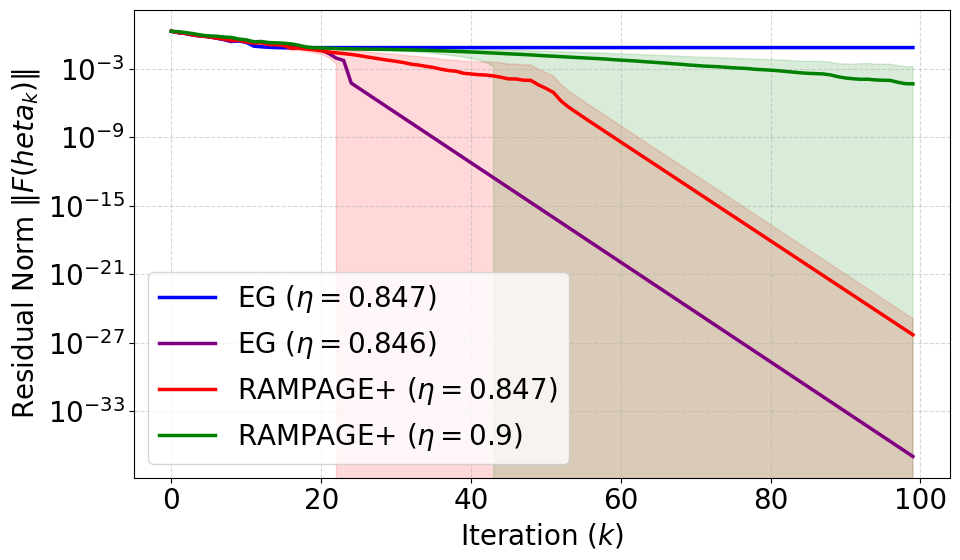}
 \caption{\footnotesize 2D High-Frequency Game: Residual}
\end{subfigure}

\medskip  

\begin{subfigure}[t]{0.32\textwidth}
 \centering
 \includegraphics[width=\textwidth]{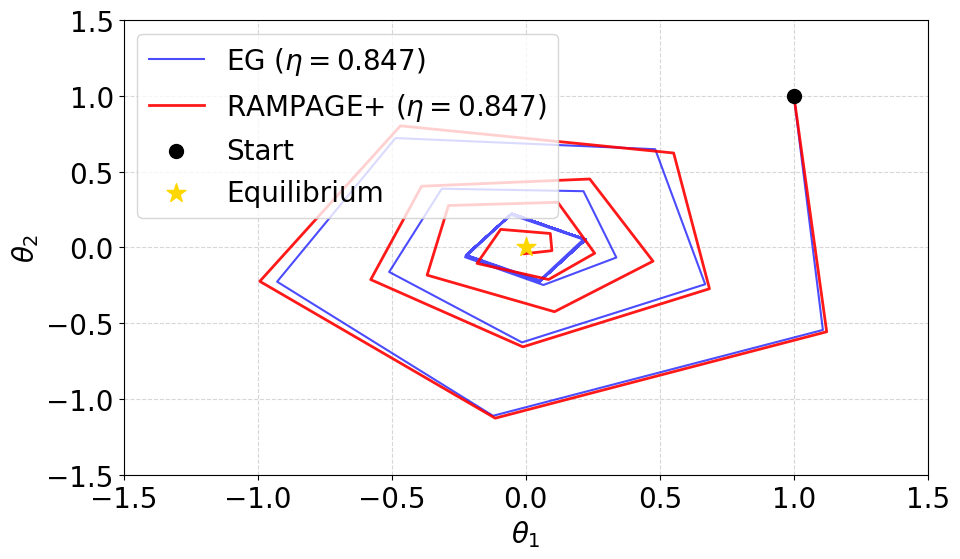}
 \caption{\footnotesize 2D High-Frequency Game: Trajectories}
\end{subfigure}\hfill
\begin{subfigure}[t]{0.32\textwidth}
 \centering
 \includegraphics[width=\textwidth]{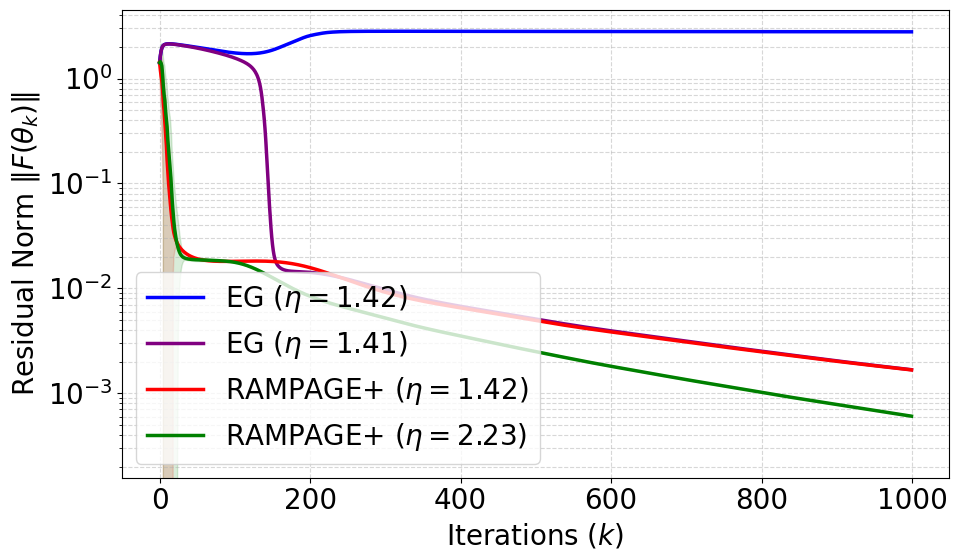}
 \caption{\footnotesize Distributionally Robust Optimization}
\end{subfigure}\hfill
\begin{subfigure}[t]{0.32\textwidth}
 \centering
 \includegraphics[width=\textwidth]{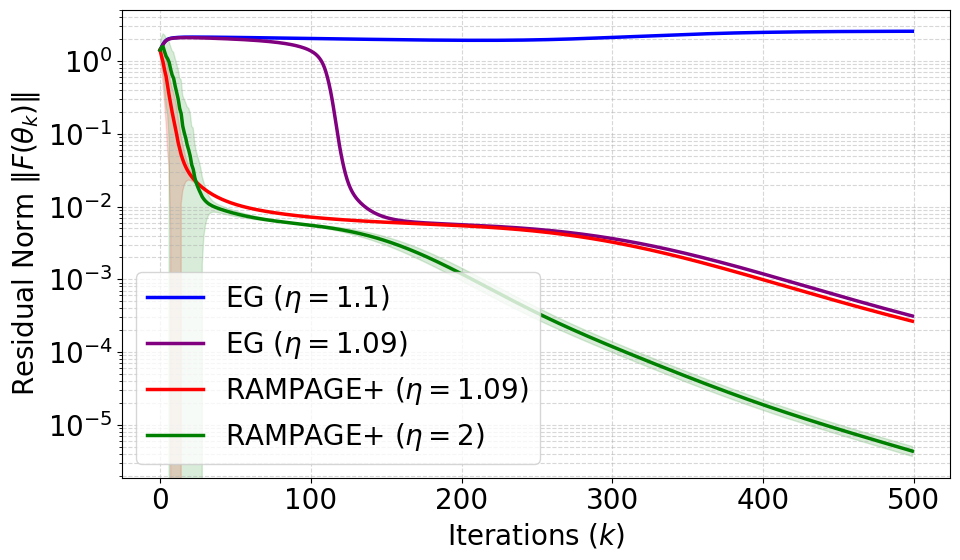}
 \caption{\footnotesize Adversarial Training}
\end{subfigure}

\caption{\small Comparison of \eqref{eq:eg} and \eqref{eq:rampage+}. See Appendix \ref{sec:exp} for details. (a) denotes an unconstrained optimization task with a nonconvex 4th-order polynomial objective, (b) denotes a nonconvex-nonconcave min-max game involving high-frequency sinusoids (motivated by infinite-width GANs), (c) and (d) denote a 2-dimensional nonconvex-nonconcave min-max game involving high-frequency sinusoids (again motivated by infinite-width GANs), and (e) and (f) correspond to 30-dimensional logistic regression tasks on the Breast Cancer Wisconsin dataset~\cite{breast_cancer_wisconsin_diagnostic_17} under KL-regularized distributionally robust optimization and adversarial training formulations, respectively. In all settings, we find two stepsizes for \eqref{eq:eg} on the edge of stability. The chosen larger stepsize makes \eqref{eq:eg} unstable ({\color{blue} blue curves}), consistent with the discretization-bias mechanism discussed above, while \eqref{eq:rampage+} enjoys convergence ({\color{red} red curves}). Furthermore, \eqref{eq:rampage+} allows larger stepsizes due to its higher stability ({\color{green} green curves}).}
\label{fig:main}
\vspace{-1em}
\end{figure*}


To circumvent the incurred bias of \eqref{eq:eg}, we motivate the use of \emph{stochastic approximation} and introduce RAndomized Mid-Point for debiAsed Gradient Extrapolation (RAMPAGE). By drawing a single random variable $u_t \sim \mathrm{Unif}([0, 1])$, the method evaluates the operator randomly along the exploration radius via $F_{\text{RAMPAGE}} = F\big(\theta_t - 2\eta u_t F(\theta_t)\big).$
Because the expectation over the uniform distribution exactly mirrors the integral $\int_0^1 ds$, in expectation $F_{\text{RAMPAGE}}$ acts as our ideal smoother
\begin{equation} \label{eq:intro_rampage_unbiased}
\mathbb{E}_{u_t} \big[ F_{\text{RAMPAGE}} \big] = \int_0^1 1 \cdot F\big(\theta_t - 2\eta s F(\theta_t)\big) ds = \mathcal{I}_{\gamma}.
\end{equation}
\eqref{eq:rampage} provides an unbiased estimator of the ideal low-pass filter $\mathcal{I}_{\gamma}$. However, evaluating a single random sample injects some degree of variance (see Appendix \ref{app:variance}).
To resolve this issue, we propose a variance-reduced variant, \eqref{eq:rampage+}. We maintain the unbiased integration while controlling the variance by drawing an antithetic pair of evaluations. Utilizing $u_t \sim \mathrm{Unif}([0,1])$ and its deterministic, antithetic complement $\tilde{u}_t = 1 - u_t$, the \eqref{eq:rampage+} update field is defined by
\begin{equation} \label{eq:intro_rampage_plus_update}
\bar{F}_t = \frac{1}{2} F\big(\theta_t - 2\eta u_t F(\theta_t)\big) + \frac{1}{2} F\big(\theta_t - 2\eta \tilde{u}_t F(\theta_t)\big).
\end{equation}
Given that $\tilde{u}_t = 1-u_t$ and both variables are uniformly distributed on $[0,1]$, \eqref{eq:rampage+} remains an unbiased estimator of $\mathcal{I}_{\gamma}$. The antithetic pairing also cancels the first-order Taylor fluctuation along the exploration segment, thereby reducing the leading-order estimation error compared to both the single-sample estimator and \eqref{eq:eg}; see Appendix \ref{app:variance} for details.

\begin{proposition} \label{prop:estimation_error1}
Let $\mathcal{I}_{\gamma} = \int_0^1 F(\theta_t - 2\eta s F_t) ds$ define the exact continuous-time average. Then, upon defining the estimation error as Bias$^2$+Variance, the leading term for \eqref{eq:eg} is $\frac{1}{36}\eta^4 \|H_t[F_t, F_t]\|^2$, while the leading term for \eqref{eq:rampage+} is reduced to $\frac{1}{45}\eta^4 \|H_t[F_t, F_t]\|^2$.
\end{proposition}
Therefore, one expects \eqref{eq:rampage+} to enjoy improved stability and to allow larger stepsizes compared to \eqref{eq:eg} in nonlinear regimes. 

\subsection{Numerical Verification}
To demonstrate the effectiveness of the proposed unbiased method in terms of its stability, our numerical results in Appendix \ref{sec:exp} consider a conservative case (corresponding to nonconvex optimization; see Appendix \ref{sec:exp-1}) and two non-conservative cases (corresponding to min-max games motivated by infinite-width GANs; see Appendices \ref{sec:exp-2} and \ref{sec:exp-3}), with results shown in Fig. \ref{fig:main}. For each case, we run two instances of \eqref{eq:eg} with slightly different stepsizes chosen at the edge of instability, one leads to convergence and the other leads to divergence due to the discretization bias we discussed. We use the latter, unstable stepsize for the proposed method as well. As the figure demonstrates, \eqref{eq:eg} becomes unstable at the larger stepsize, consistent with the discretization-bias mechanism discussed above, while \eqref{eq:rampage+} enjoys fast convergence, consistent with its unbiased path-averaging interpretation. Furthermore, \eqref{eq:rampage+} by using antithetic sampling enjoys a significantly low variance, as we explained. We also note that in all cases \eqref{eq:rampage+} still converges when using a larger stepsize than \eqref{eq:eg} (around 6\%-26\% larger). Therefore, \eqref{eq:rampage+} could lead to further stabilization for problems with nonlinear fields.

Additionally, Fig. \ref{fig:main} (e) and (f) correspond to standard logistic regression tasks on the Breast Cancer Wisconsin dataset~\cite{breast_cancer_wisconsin_diagnostic_17} under KL-regularized distributionally robust optimization and adversarial training formulations, respectively, both of which are inherently instances of minmax games (see Appendices \ref{sec:exp-4} and \ref{sec:exp-5} for detail). As the figure show, \eqref{eq:rampage+} can tolerate larger stepsizes and converge faster than \eqref{eq:eg} as a result, corroborating the stabilizing nature of the proposed stochastic approximation with variance reduced antithetic sampling.

\subsection{Theoretical Contributions}
Besides the proposed idea of leveraging stochastic integration with antithetic sampling, we provide convergence guarantees for \eqref{eq:rampage} and \eqref{eq:rampage+} in a number of settings. In particular, our key theoretical contributions are as follows:
\begin{itemize}[leftmargin=*]
 \item \textbf{Root-Finding Problems:} We establish $\mathcal{O}(1/k)$ best-iterate convergence rates for \eqref{eq:rampage} and \eqref{eq:rampage+} under co-coercivity, co-hypomonotonicity, and generalized Lipschitzness.
 \item \textbf{Variational Inequality:} We propose the symmetrically scaled variants of both methods and show their convergence under Lipschitzness and monotonicity for variational inequality problems.
 \item \textbf{Ergodic Rates for Convex-Concave Games:} We extend the analysis to unconstrained smooth convex-concave minimax games, proving that the ergodic sequence of the randomized updates converges to a Nash equilibrium at a rate of $\mathcal{O}(1/k)$, with \eqref{eq:rampage+} achieving this bound purely deterministically despite its stochastic nature. We further extend our results to stochastic unconstrained smooth convex-concave minimax games.
\end{itemize}
\section{Preliminaries and Problem Formulation}
\noindent\textbf{Root-Finding Problems:} The most fundamental unconstrained equilibrium framework is the root-finding problem, which seeks an optimal state vector $\theta^* \in \mathbb{R}^p$ residing in the zero set of the operator. This is formally defined by the algebraic condition
\begin{equation}\label{eq:rf}
\tag{RF}
F(\theta^*) = 0.
\end{equation}
We assume that the equilibrium set is non-empty, denoting an arbitrary optimal state as $\theta^* \in \mathrm{zer}(F)$, satisfying $F(\theta^*) = 0$.

\noindent\textbf{Variational Inequality Problems:} When the state space is defined by hard boundary conditions to be a non-empty, closed, and convex feasible set $\mathcal{X} \subset \mathbb{R}^p$, the unconstrained root-finding formulation is naturally replaced by the classic Stampacchia VI. A point $\theta^* \in \mathcal{X}$ is a solution to the Variational Inequality if it satisfies
\begin{equation} \label{eq:vi}
\tag{VI}
\langle F(\theta^*), \theta - \theta^* \rangle \ge 0, \quad \forall \theta \in \mathcal{X}.
\end{equation}
This framework encompasses a vast array of optimization paradigms, including standard convex minimization and smooth convex-concave minimax games. We refer the reader to Appendix \ref{sec:prior} for a detailed discussion of related work.

We state the following assumption, which formally defines certain functional properties leveraged by the prior work for the analysis of EG and its variants \cite{combettes2004proximal,bauschke2020generalized,chen2023generalized,vankov2024adaptive,tran2023sublinear}.

\begin{assumption}\label{ass:func}
Consider two arbitrary $\theta, \theta' \in \mathcal{X}$. An operator $F: \mathbb{R}^p \to \mathbb{R}^p$ is
\begin{enumerate}[leftmargin=*]
 \item $L$-Lipschitz continuous if $\exists\; L > 0$ such that $\|F(\theta) - F(\theta')\| \le L \|\theta - \theta'\|$.
 \item $\alpha$-symmetric $(L_0, L_1)$-Lipschitz if for some $L_0, L_1 \geq 0$ and $\alpha \in (0, 1]$,
 \begin{equation}
  \| F(\theta) - F(\theta')\| \leq \left( L_0 + L_1 \max_{\gamma \in [0, 1]} \left\|F(\gamma \theta + ( 1- \gamma) \theta') \right\|^{\alpha} \right) \| \theta- \theta'\|.
 \end{equation}
\item monotone if $\langle F(\theta) - F(\theta'), \theta - \theta' \rangle \ge 0.$
\item $\mu$-co-coercive if $\exists \; \mu > 0$ such that $\langle F(\theta) - F(\theta'), \theta - \theta' \rangle \ge \mu \|F(\theta) - F(\theta')\|^2.$
\item $\rho$-co-hypomonotone if $\exists \;\rho \geq 0$ such that $\langle F(\theta) - F(\theta'), \theta - \theta' \rangle \ge -\rho \|F(\theta) - F(\theta')\|^2.$
\end{enumerate}
\end{assumption}
While co-coercivity is satisfied by the gradients of smooth, convex objective functions and some non-conservative fields, it is violated by the rotational dynamics intrinsic to min-max games. On the other hand,
every $L$-Lipschitz monotone operator is $\rho$-co-hypomonotone with $\rho = 0$. Also, a $\mu$-co-coercive operator is monotone and $\frac{1}{\mu}$-Lipschitz.

We state the following proposition for $\alpha$-symmetric $(L_0, L_1)$-Lipschitz operators \cite{chen2023generalized}, which bas been recently used to analyze \eqref{eq:eg} and related methods \cite{vankov2024adaptive,pmlr-v235-vankov24a,choudhuryextragradient}.
\begin{proposition} \label{prop:equiv_formulation}
 Suppose $F$ is an $\alpha$-symmetric $(L_0, L_1)$-Lipschitz operator. Then, for $\alpha \in (0, 1)$ we have
 \begin{equation}\label{eq:alpha(0,1)}
 \textstyle
  \| F(\theta) - F(\theta') \| \leq \left(K_0 + K_1 \|F(\theta)\|^{\alpha} + K_2 \|\theta - \theta'\|^{\frac{\alpha}{1 - \alpha}} \right) \|\theta - \theta'\|,
 \end{equation}
 where $K_0 = L_0 (2^{\frac{\alpha^2}{1 - \alpha}} + 1)$, $K_1 = L_1 \cdot 2^{\frac{\alpha^2}{1 - \alpha}}$ and $K_2 = L_1^{\frac{1}{1 - \alpha}} \cdot 2^{\frac{\alpha^2}{1 - \alpha}} \cdot 3^{\alpha} (1 - \alpha)^{\frac{\alpha}{1 - \alpha}}$.
\end{proposition}

Let $\mathcal{B} \subset \mathcal{X}$ be a closed and bounded set containing the optimal state $\theta^*$. The restricted gap function evaluated at a state $\theta \in \mathcal{X}$ over the domain $\mathcal{B}$ is defined as $\operatorname{Gap}_{\mathcal{B}}(\theta) = \sup_{v \in \mathcal{B}} \langle F(\theta), \theta - v \rangle$. The following standard lemma \cite{Konnov2001,Facchinei2003,nesterov2007dual} is typically used to establish convergence results in terms of $\operatorname{Gap}$ by leveraging results established in terms of the residual norm.
\begin{lemma}\label{lem:gap}
 Let $z_l = \Pi_{\mathcal{X}}(\theta_l - \eta F(\theta_l))$ and the residual $r_l = \theta_l - z_l$. It holds that
 \begin{equation}
\operatorname{Gap}_{\mathcal{B}}(\theta_l) =\sup_{v \in \mathcal{B}} \langle F(\theta_l), \theta_l - v \rangle\le \|r_l\| \left( \|F(\theta_l)\| + \frac{1}{\eta} \sup_{v \in \mathcal{B}} \|\theta_l - v\| \right).
\end{equation}
\end{lemma}
\section{RAMPAGE+ and Its Analyses}\label{sec:rampage+}
In this section, we analyze our main method, \eqref{eq:rampage+}, in numerous settings while the analyses of \eqref{eq:rampage} are presented in Appendix \ref{sec:rampage}.

At each iteration $t \ge 0$, given a base iterate $\theta_t \in \mathbb{R}^p$ and a step size $\eta > 0$, we draw a single uniform random variable $u_t \sim \mathrm{Unif}([0, 1])$. Then, \eqref{eq:rampage} is formally defined by
\begin{equation} \label{eq:rampage}
\tag{RAMPAGE}
\begin{aligned}
y_t &= \theta_t - 2\eta u_t F(\theta_t), \qquad
\theta_{t+1} = \theta_t - \eta F(y_t).
\end{aligned}
\end{equation}

\eqref{eq:rampage+} on the other hand generates an antithetic pair of intermediate exploration states, utilizing a single random scalar $u_t \sim \mathrm{Unif}([0,1])$ and its deterministic complement $\tilde{u}_t = 1 - u_t$. At each iteration $t \ge 0$, the update sequence is formulated as
\begin{equation} \label{eq:rampage+}
\tag{RAMPAGE+}
\begin{aligned}
y_t &= \theta_t - 2\eta u_t F(\theta_t), \qquad
\tilde{y}_t = \theta_t - 2\eta \tilde{u}_t F(\theta_t), \\
\bar{F}_t &= \frac{1}{2} \big( F(y_t) + F(\tilde{y}_t) \big), \qquad
\theta_{t+1} = \theta_t - \eta \bar{F}_t.
\end{aligned}
\end{equation}
Note that with $u_t = \tilde{u}_t = 1/2$, both methods reduce to \eqref{eq:eg}.
\subsection{Root-Finding Problems}
We first consider \eqref{eq:rf} with an operator $F: \mathbb{R}^p \to \mathbb{R}^p$ that is both $L$-Lipschitz continuous and $\mu$-co-coercive. 

\begin{theorem} \label{thm:esramp_cocoercive}
Let $F: \mathbb{R}^p \to \mathbb{R}^p$ be an $L$-Lipschitz continuous and $\mu$-co-coercive operator with $\mu > 0$. Assume $\mathrm{zer}(F) \neq \emptyset$. Let $\{\theta_t\}_{t \ge 0}$ denote the sequence generated by \eqref{eq:rampage+}. If the step size satisfies $0 < \eta <\frac{1}{\sqrt{3}L}$ with $\eta \leq 2\mu$, then for any $\theta^* \in \mathrm{zer}(F)$,
\begin{equation}
\min_{0 \le l \le k} \|F(\theta_l)\|^2 \le \frac{1}{k+1} \sum_{l=0}^k \|F(\theta_l)\|^2 \le \frac{\|\theta_0 - \theta^*\|^2}{C_{RA+} (k+1)},
\end{equation}
where $C_{RA+} = \eta^2 \left( 1 - 3\eta^2 L^2 \right) > 0$.
\end{theorem}
The proof is available in Appendix \ref{sec:esramp_cocoercive}. The result thus demonstrates the standard $\mathcal{O}(1/k)$ convergence for this setting. Furthermore, as we discuss further in Remark \ref{rem:var}, \eqref{eq:rampage+} enjoys a lower variance compared to \eqref{eq:rampage}, as it benefits from variance reduction properties of antithetic sampling. As we can see, compared to \eqref{eq:rampage}, by using antithetic sampling, \eqref{eq:rampage+} enjoys a purely deterministic convergence despite being a randomized algorithm. However, we could improve the requirement on $\eta$ at the cost of stating a result that holds on expectation.

Next, we turn our attention to \eqref{eq:rf} under co-hypomonotonicity and Lipschitzness.
\begin{theorem} \label{thm:es_ramp_general}
Let $F: \mathbb{R}^p \to \mathbb{R}^p$ be an $L$-Lipschitz continuous and $\rho$-co-hypomonotone operator with $\rho \ge 0$. Assume $\mathrm{zer}(F) \neq \emptyset$. Let $\{\theta_t\}_{t \ge 0}$ denote the sequence generated by \eqref{eq:rampage+}. If the step size $\eta$ is chosen such that
\begin{equation}
C_{RA+}^{\rho,\eta} := \frac{3}{4} - \frac{2\rho}{\eta} - 16L^2 \left(\rho + \frac{\eta}{2}\right)^2 - 4L^2\eta \left(\rho + \frac{\eta}{2}\right) > 0,
\end{equation}
then for any $\theta^* \in \mathrm{zer}(F)$,
\begin{equation}
\min_{0 \le l \le k} \|F(\theta_l)\|^2 \le \frac{1}{k+1} \sum_{l=0}^k \|F(\theta_l)\|^2 \le \frac{\|\theta_0 - \theta^*\|^2}{C_{RA+}^{\rho,\eta} (k+1)}.
\end{equation}
\end{theorem}
The proof is available in Appendix \ref{sec:es_ramp_general}. As we can see by using antithetic sampling, \eqref{eq:rampage+} enjoys a purely deterministic convergence despite being a randomized algorithm.

When $\rho=0$, the operator is monotone. In this case, with some calculations we obtain $C_{RA+}^{0,\eta} = \frac{3}{4} - 6L^2\eta^2$ and the upper bound $\eta < \frac{\sqrt{2}}{4L}$. 

\begin{remark}\label{rem:var}
The benefit of \eqref{eq:rampage+} over the single-sample \eqref{eq:rampage} method is seen by examining both the probabilistic nature and the absolute magnitude of their respective descent multipliers, $C_{RA+}^{\rho,\eta}$ and $C_{RA}^{\rho,\eta}$ (see Theorem \ref{thm:es_rmp_convergence} in Appendix \ref{sec:rampage}). A key distinction lies in the realization dependence of the trajectory. The bounds derived for \eqref{eq:rampage} are valid in expectation. Conversely, by perfectly correlating the antithetic exploration samples, \eqref{eq:rampage+} removes the variance completely.

Beyond the transition from expected to deterministic convergence, \eqref{eq:rampage+} enjoys a faster descent. This is immediately evident in the baseline constants: \eqref{eq:rampage+} establishes a primary baseline descent of $3/4$, representing a $50\%$ improvement over the $1/2$ baseline extracted by the single-sample \eqref{eq:rampage} formulation. To concretely quantify this acceleration, we evaluate both multipliers in the pure monotone limit where $\rho = 0$. Under this regime, the \eqref{eq:rampage} multiplier reduces to $C_{RA}^{0,\eta} = \frac{1}{2} - 4L^2\eta^2$, while the \eqref{eq:rampage+} multiplier evaluates to $C_{RA+}^{0,\eta} = \frac{3}{4} - 6L^2\eta^2$. Both dictate the maximum step size threshold $\eta < \frac{1}{2\sqrt{2}L}$ to ensure absolute positivity. However, for any fixed step size selected within this valid domain, $C_{RA+}^{0,\eta} - C_{RA}^{0,\eta} = \frac{1}{4} - 2L^2\eta^2 > 0$. 
\end{remark}

Finally, we extend the analysis to \eqref{eq:rf} with monotone and $\alpha$-symmetric $(L_0, L_1)$-Lipschitz operators.
\begin{theorem} \label{thm:rampage_plus_alpha_monotone}
Let $F: \mathbb{R}^p \to \mathbb{R}^p$ be a monotone and $\alpha$-symmetric $(L_0, L_1)$-Lipschitz operator with $\alpha \in (0, 1)$. Assume $\mathrm{zer}(F) \neq \emptyset$. Let $\{\theta_t\}_{t \ge 0}$ denote the sequence generated by \eqref{eq:rampage+}. If the step size is set to $\eta_t = \frac{\nu}{K_0 + C_\alpha \|F(\theta_t)\|^\alpha}$ with $C_\alpha = K_1 + 2^{\frac{\alpha}{1-\alpha}} K_2$ and an appropriate absolute constant $\nu > 0$ such that
\begin{equation}
\nu \le \min \left\{ \frac{\sqrt{6} C_\alpha}{16(C_\alpha + K_1)}, \ C_\alpha \left( \frac{\sqrt{6}}{16(C_\alpha - K_1)} \right)^{1-\alpha} \right\},
\end{equation}
then for any $\theta^* \in \mathrm{zer}(F)$,
\begin{equation}
\min_{0 \le l \le k} \mathbb{E}\left[ \eta_l^2 \|F(\theta_l)\|^2 \right] \le \frac{1}{k+1} \sum_{l=0}^k \mathbb{E}\left[ \eta_l^2 \|F(\theta_l)\|^2 \right] \le \frac{4\|\theta_0 - \theta^*\|^2}{k+1}.
\end{equation}
Consequently, the subsequence of best iterates $l^\ast_k$ enjoys $\|F(\theta_{l^\ast_k})\| \xrightarrow{p} 0$.
\end{theorem}
The proof is available in Appendix \ref{sec:rampage_plus_alpha_monotone}. Note that when $L_1 = 0$ and $L_0 = L$, the operator is $L$-Lipschitz in the regular sense. By Proposition \ref{prop:equiv_formulation}, $K_1 = K_2 = 0$ and $K_0 =L$. Thus, $C_\alpha = 0$ and after some calculations we find that $\eta \le \frac{\sqrt{6}}{16 L}.$ Thus, compared to $\eta < \frac{\sqrt{2}}{4L}$ which we established for Lipschitz operators directly in Theorem \ref{thm:es_ramp_general}, the upperbound is slightly more restrictive. In the general case, we observe that the restriction on $\nu$ is relaxed compared to the result for \eqref{eq:rampage} (see Theorem \ref{thm:rampage_alpha_monotone} in Appendix \ref{sec:rampage}), as expected. 
\begin{remark}\label{rem:alpha_rampage}
 Note that for $\alpha \in (0,0.5]$, the function $M(x) = \left( \frac{\nu x}{K_0 + C_\alpha x^\alpha} \right)^2$ with $x = \|F(\theta_l)\| \geq 0$ is convex such that by Jensen's inequality, $M(\E[\|F(\theta_l)\|]) \leq \mathbb{E}[M(\|F(\theta_l)\|)]$. As $M$ is monotone and invertible, using the result of Theorem \ref{thm:rampage_plus_alpha_monotone} we obtain
\begin{equation}
 \min_{0 \le l \le k} \mathbb{E}[\|F(\theta_l)\|] \le M^{-1}\left(\frac{4\|\theta_0 - \theta^*\|^2}{k+1}\right).
\end{equation}
The inverse $ M^{-1}$, despite being unique, does not generally have a closed-form expression. For $\alpha = 0.5$, however, the closed-form expression amounts to $M^{-1}(y) = \left( \frac{C_{1/2} \sqrt{y} + \sqrt{C_{1/2}^2 y + 4\nu K_0 \sqrt{y}}}{2\nu} \right)^2$, and we obtain the bound by using $(p+\sqrt{p^2+q^2})^2=4p^2+2q^2-(p-\sqrt{p^2+q^2})^2\leq 4p^2+2q^2$
\begin{equation}
  \min_{0 \le l \le k} \mathbb{E}[\|F(\theta_l)\|] \le \frac{4 C_{1/2}^2\|\theta_0 - \theta^*\|^2}{\nu^{2}(k+1)} + \frac{2 K_0}{\nu} \sqrt{\frac{4\|\theta_0 - \theta^*\|^2}{k+1}}.
\end{equation}
\end{remark}
\subsection{Monotone Variational Inequality Problems and SS-RAMPAGE+}
Transitioning from unconstrained equilibrium problems to constrained \eqref{eq:vi} introduces an obstacle in the proof. The asymmetrical scaling in \eqref{eq:rampage+} leads to ineffective bounds when the trajectory interacts with the boundary $\partial \mathcal{X}$. Thus, to guarantee expected improvement, the randomized scaling must be symmetrically coupled to both the intermediate exploration projection and the final update projection. This symmetric application ensures that the polarization identities used in the proof generated by the intermediate state cancel the expansive distance metrics induced by the outer projection mapping.
Thus, we propose the following Symmetrically Scaled (SS) variants. At each iteration $t \ge 0$, given a base iterate $\theta_t \in \mathcal{X}$ and a step size $\eta > 0$, we draw a single uniform random variable $u_t \sim \mathrm{Unif}([0, 1])$ and execute
\begin{equation}\label{eq:ss-rampage}
\tag{SS-RAMPAGE}
\begin{aligned}
y_t = \Pi_{\mathcal{X}} \big( \theta_t - 2\eta u_t F(\theta_t) \big), \quad
\theta_{t+1} = \Pi_{\mathcal{X}} \big( \theta_t - 2\eta u_t F(y_t) \big).
\end{aligned}
\end{equation}
Because $\mathcal{X}$ is a closed and convex set, the projection operator ensures the trajectory maintains feasibility, guaranteeing $y_t, \theta_{t+1} \in \mathcal{X}$. 

To leverage antithetic sampling on the other hand, we draw $u_t \sim \mathrm{Unif}([0,1])$ and set $\tilde{u}_t=1-u_t$ to define the antithetic intermediate states, followed by the averaged primary update
\begin{equation}\label{eq:ss-rampage+}
\tag{SS-RAMPAGE+}
\begin{aligned}
y_t = \Pi_{\mathcal{X}}\big(\theta_t - 2\eta u_t F(\theta_t)\big), \;
\tilde{y}_t = \Pi_{\mathcal{X}}\big(\theta_t - 2\eta \tilde{u}_t F(\theta_t)\big), \;
\theta_{t+1} = \Pi_{\mathcal{X}}\big(\theta_t - \eta u_t F(y_t) - \eta \tilde{u}_t F(\tilde{y}_t)\big).
\end{aligned}
\end{equation}
Note that when $u_t = \tilde{u}_t = 1/2$, both methods readily reduce to the constrained version of \eqref{eq:eg}, but the symmetrically scaled methods differ from their unconstrained updates when $\mathcal{X}=\mathbb{R}^p$. 

We now state the convergence of \eqref{eq:ss-rampage+} while that of \eqref{eq:ss-rampage} is presented in Appendix \ref{sec:rampage}.
\begin{theorem} \label{thm:cp_ramp_monotone}
Let $F: \mathbb{R}^p \to \mathbb{R}^p$ be an $L$-Lipschitz continuous and monotone operator. Assume the Variational Inequality admits a solution $\theta^* \in \mathcal{X}$. Let $\{\theta_t\}_{t \ge 0}$ denote the sequence generated by \eqref{eq:ss-rampage+}. If the step size satisfies $0 < \eta < \frac{1}{2L}$, then
\begin{equation}
\min_{0 \le l \le k} \left( \frac{\|\theta_l - y_l\|^2 + \|\theta_l - \tilde{y}_l\|^2}{2} \right) \le \frac{1}{2(k+1)} \sum_{l=0}^k \left( \|\theta_l - y_l\|^2 + \|\theta_l - \tilde{y}_l\|^2 \right) \le \frac{\|\theta_0 - \theta^*\|^2}{C_{SS+} (k+1)},
\end{equation}
where $C_{SS+} = 1 - 4\eta^2 L^2 > 0$.
\end{theorem}
The proof is available in Appendix \ref{sec:cp_ramp_monotone}. We also stated a result in terms of the restricted gap.
\begin{corollary} \label{cor:best_iterate_gap}
Let the conditions of Theorem \ref{thm:cp_ramp_monotone} hold. Assume that the sequences $\{\theta_l\}_{l=0}^k$ and $\{F(\theta_l)\}_{l=0}^k$ are bounded such that $\sup_{l \ge 0, v \in \mathcal{B}} \|\theta_l - v\| \le D_{\mathcal{B}}$ and $\sup_{l \ge 0} \|F(\theta_l)\| \le G_{\mathcal{B}}$. Let $l^* = \mathrm{argmin}_{0 \le l \le k} \left( \|\theta_l - y_l\|^2 + \|\theta_l - \tilde{y}_l\|^2 \right)$ denote the index of the best iterate generated by \eqref{eq:ss-rampage+} up to iteration $k$. Then, by Lemma \ref{lem:gap}
\begin{equation}
\operatorname{Gap}_{\mathcal{B}}(\theta_{l^*}) \le \left( G_{\mathcal{B}} + \frac{D_{\mathcal{B}}}{\eta} \right) \frac{\sqrt{2}\|\theta_0 - \theta^*\|}{\sqrt{\big(1 - 4\eta^2 L^2\big)(k+1)}}.
\end{equation}
\end{corollary}
See Appendix \ref{sec:best_iterate_gap} for the proof. Once again, we observe that by using antithetic sampling, \eqref{eq:rampage+} enjoys a purely deterministic convergence despite being a randomized algorithm.
\subsection{Application to Min-Max Games}\label{sec:games}
We now apply the \eqref{eq:rampage} to min-max optimization, specifically the unconstrained smooth convex-concave game. We will consider both deterministic and stochastic game settings.
\subsubsection{Deterministic Games}
Let the joint state variable be $\theta = (x, z) \in \mathbb{R}^n \times \mathbb{R}^m$. We consider the classical minimax formulation
\begin{equation} \label{eq:min-max_problem}
\min_{x \in \mathbb{R}^n} \max_{z \in \mathbb{R}^m} f(x, z),
\end{equation}
where $f(x, z)$ is continuously differentiable, convex in $x$, and concave in $z$. This induces the vector field operator $F: \mathbb{R}^{n+m} \to \mathbb{R}^{n+m}$
\begin{equation} \label{eq:operator_def}
F(\theta) = \begin{bmatrix} \nabla_x f(x, z) \\ -\nabla_z f(x, z) \end{bmatrix}.
\end{equation}
By the convexity-concavity of $f$, $F$ is monotone. For any intermediate point $y_t = (x_t^m, z_t^m)$ and reference state $\theta = (x, z)$, we have
\begin{equation} \label{eq:convex_concave_bound}
f(x_t^m, z) - f(x, z_t^m) \le \langle \nabla_x f(x_t^m, z_t^m), x_t^m - x \rangle - \langle \nabla_z f(x_t^m, z_t^m), z_t^m - z \rangle = \langle F(y_t), y_t - \theta \rangle.
\end{equation}
Next, we analyze the proposed \eqref{eq:rampage+} method for the deterministic games.
\begin{theorem} \label{thm:ergodic_esramp}
Let $F$ be the $L$-Lipschitz continuous monotone operator associated with the smooth convex-concave function $f$. Let $\{\theta_t\}_{t=0}^k$ be the primary sequence, and $\{y_t\}_{t=0}^k$, $\{\tilde{y}_t\}_{t=0}^k$ be the antithetic intermediate sequences generated by \eqref{eq:rampage+}. Define the ergodic average over $k$ iterations as
\begin{equation} \label{eq:ramp_ergodic_def}
\hat{\theta}_k = (\hat{x}_k, \hat{z}_k) = \frac{1}{2(k+1)} \sum_{t=0}^k (y_t + \tilde{y}_t).
\end{equation}
If the step size satisfies $0 < \eta \le \frac{\sqrt{3}-1}{2L}$, then for any reference point $\theta = (x, z)$,
\begin{equation} \label{eq:ramp_ergodic_rate}
f(\hat{x}_k, z) - f(x, \hat{z}_k) \le \frac{\|\theta_0 - \theta\|^2}{2\eta (k+1)}.
\end{equation}
\end{theorem}
The proof is available in Appendix \ref{sec:ergodic_esramp}. As we can see by using antithetic sampling, \eqref{eq:rampage+} enjoys a purely deterministic convergence despite being a randomized algorithm. Furthermore, we crucially define the solution as the average of both $y_t$ and $\tilde{y}_t$ to establish the ergodic convergence.
\subsubsection{Stochastic Convex-Concave Games}
In data-driven environments, the exact vector field $F(\theta)$ is computationally inaccessible. Instead, the algorithm queries an estimator $\hat{F}(\theta, \xi)$ driven by a random variable $\xi$ drawn from an underlying data distribution $\mathcal{D}$. We impose standard unbiasedness and bounded variance constraints on this Stochastic First-Order Oracle (SFO) in our study of stochastic min-max games. 
\begin{assumption}\label{ass:sfo}
For any fixed state $\theta \in \mathbb{R}^p$, SFO returns an estimator $\hat{F}(\theta, \xi)$ such that
\begin{equation}
\mathbb{E}_{\xi \sim \mathcal{D}} \big[ \hat{F}(\theta, \xi) \big] = F(\theta), \qquad \mathbb{E}_{\xi \sim \mathcal{D}} \left[ \|\hat{F}(\theta, \xi) - F(\theta)\|^2 \right] \le \sigma^2, \quad \text{for some }\sigma \ge 0.
\end{equation}
\end{assumption}

We now extend our analysis to stochastic games under Assumption \ref{ass:sfo} and present the corresponding result for \eqref{eq:rampage+}.
\begin{theorem} \label{thm:sfo_ergodic_rampage_plus}
Let $F$ be the $L$-Lipschitz continuous monotone operator associated with the smooth convex-concave function $f$. Let $\{\theta_t\}_{t=0}^k$ and the intermediate pairs $\{y_t\}_{t=0}^k, \{\tilde{y}_t\}_{t=0}^k$ be the sequences generated by the SFO \eqref{eq:rampage+} method. Define the ergodic average over $k$ iterations as
\begin{equation} \label{eq:sfo_ramp_ergodic_def}
\hat{\theta}_k = (\hat{x}_k, \hat{z}_k) = \frac{1}{2(k+1)} \sum_{t=0}^k (y_t + \tilde{y}_t).
\end{equation}
If the constant step size satisfies $0 < \eta \le \frac{\sqrt{3}}{2L\sqrt{C_{R+}}}$, where $C_{R+} = \frac{22}{3}$, then for any reference point $\theta = (x, z)$,
\begin{equation}
\mathbb{E} \left[ f(\hat{x}_k, z) - f(x, \hat{z}_k) \right] \le \frac{\|\theta_0 - \theta\|^2}{2\eta (k+1)} + \frac{\eta \left( \frac{3}{2} + \eta^2 L^2 C_{R+} \right)}{2} \sigma^2.
\end{equation}
\end{theorem}
The proof is available in Appendix \ref{sec:sfo_ergodic_rampage_plus}. Compared to \eqref{eq:rampage}'s dominant noise floor (see Theorem \ref{thm:sfo_ergodic_rampage} in Appendix \ref{sec:rampage}), i.e., $\eta \sigma^2$, the dominant noise floor of \eqref{eq:rampage+} is reduced to $0.75\eta \sigma^2$, further demonstrating the variance reduction properties of leveraging antithetic sampling.
\section{Conclusion}
In this work, we addressed the discretization bias inherent in EG and its variants when applied to non-linear vector fields. To overcome these integration shortcomings, we introduced RAMPAGE, a randomized unbiased gradient extrapolation scheme, and its variance-reduced counterpart, RAMPAGE+. By leveraging exact antithetic sampling, RAMPAGE+ functions as an unbiased geometric path-integrator that mitigates internal first-order variance, providing theoretical improvements over the base RAMPAGE architecture. We established optimal $\mathcal{O}(1/k)$ best-iterate convergence guarantees for numerous root-finding and constrained VI problems. Notably, our analysis demonstrated that, despite its randomized exploration mechanism, the antithetic coupling in RAMPAGE+ yields purely deterministic bounds in several of the studied settings.
\bibliographystyle{acm}
\bibliography{ref.bib}
\newpage
\appendix

\section*{Appendix Table of Contents}
\addcontentsline{toc}{section}{Appendix}
\markboth{Appendix}{Appendix}
\startcontents[appendix]
\printcontents[appendix]{l}{1}{\setcounter{tocdepth}{3}}
\newpage
\section{Related Work}\label{sec:prior}
\noindent\textbf{Extragradient and Its Variants:}
\eqref{eq:eg}, introduced independently by \cite{Korpelevic1976} and \cite{antipin1976}, has established itself as a foundational technique for addressing saddle-point problems, VIs, and their generalizations to nonlinear inclusions \cite{Bauschke2011,Facchinei2003}. By relying on a two-step procedure evaluating the operator at an extrapolated point, \eqref{eq:eg} mitigates the divergent behavior often exhibited by standard gradient methods in non-co-coercive monotone settings \cite{Facchinei2003,tseng1995linear}. 

Over the decades, \eqref{eq:eg} has spawned a vast literature focusing on reducing its per-iteration complexity and extending its convergence guarantees. Notable among these is Popov's past-extragradient (or optimistic gradient) method \cite{popov1980modification}, which recycles the operator evaluation from the previous step to achieve a single-call per-iteration cost. Forward-backward-forward splitting \cite{tseng2000modified} provides another single-resolvent alternative. Further modifications include projection and contraction methods \cite{he1997class,xiu2001convergence}, subgradient extragradient schemes \cite{censor2011subgradient,censor2012extensions}, and hybrid approximate variants \cite{solodov1999hybrid,solodov1999new}. 

Recently, the resurgence of min-max optimization within machine learning, driven by applications such as GANs \cite{goodfellow2014generative,arjovsky2017wasserstein}, adversarial training \cite{madry2018towards}, and distributionally robust optimization \cite{namkoong2016stochastic,levy2020large} to name a few, has renewed interest in \eqref{eq:eg}. This has led to the development of numerous variants, including adaptive \cite{antonakopoulos2020adaptive}, stochastic \cite{gorbunov2022stochastic,gidel2018variational,mishchenko2020revisiting,hsieh2019convergence}, and decentralized algorithms \cite{beznosikov2022decentralized,wei2021last}.

\noindent\textbf{Beyond Monotonicity and Global Lipschitzness:}
Traditional convergence analyses for \eqref{eq:eg} rely heavily on the assumption that the operator is globally $L$-Lipschitz continuous \cite{Facchinei2003,mokhtari2020unified}. However, modern applications frequently violate this condition. Recent efforts have sought to relax both the monotonicity and Lipschitz requirements. 

For non-monotone problems, \eqref{eq:eg} and its variants have been analyzed under the weak Minty condition \cite{diakonikolas2021efficient,pethick2022escaping,pethick2023solving} and co-hypomonotonicity \cite{combettes2004proximal,luo2022last,tran2023extragradient}, establishing local or sublinear convergence rates. Besides establishing convergence under standard settings, we further study the convergence of the proposed methods under co-hypomonotonicity.

To address the limitations of the global Lipschitz assumption, \cite{zhang2019gradient} introduced the $(L_0, L_1)$-smoothness condition for minimization problems, demonstrating its relevance to training modern neural networks like LSTMs \cite{hochreiter1997long} and Transformers \cite{ahn2023linear}. This was later formalized as the condition $\|\nabla f(x) - \nabla f(y)\| \le (L_0 + L_1\|\nabla f(x)\|)\|x-y\|$ \cite{chen2023generalized,vankov2024optimizing,gorbunov2024methods}. This was subsequently extended to VIs via the $\alpha$-symmetric $(L_0, L_1)$-Lipschitz assumption \cite{pmlr-v235-vankov24a,xian2024delving,choudhuryextragradient}. Our work builds upon this refined characterization, utilizing it to design adaptive step sizes. Other notable relaxations of the Lipschitz condition include relative smoothness \cite{bauschke2017descent}, glocal smoothness \cite{fox2025glocal}, H\"{o}lder smoothness \cite{gorbunov2021high}, and star-cocoercivity \cite{loizou2021stochastic,beznosikov2023stochastic}.

\noindent\textbf{Accelerated Extragradient Methods:}
As \eqref{eq:eg} is limited to $\mathcal{O}(1/\sqrt{k})$ \cite{golowich2020last}, achieving $\mathcal{O}(1/k)$ last-iterate convergence rates and faster rates for VIs and inclusions requires one of two paradigms: Nesterov's momentum \cite{Nesterov1983,Nesterov2005c,Beck2009} or Halpern's fixed-point iteration \cite{halpern1967fixed}. 
Nesterov's acceleration, while ubiquitous in convex optimization, requires careful adaptation for monotone inclusions, often relying on proximal-point frameworks \cite{guler1992new,kim2021accelerated,attouch2020convergence} or performance estimation problem (PEP) techniques \cite{drori2014performance,ryu2020operator,gupta2022branch}. Recent continuous-time analyses have also provided insights into Nesterov-type acceleration for min-max problems \cite{Su2014,bot2022fast,shi2021understanding}.

Halpern's iteration, originally designed for nonexpansive operators, was shown by \cite{lieder2021convergence} to achieve an $\mathcal{O}(1/k)$ rate. This was extended to root-finding and VIs by \cite{diakonikolas2020halpern}. The Extra-Anchored Gradient (EAG) method \cite{yoon2021accelerated} fused Halpern iteration with \eqref{eq:eg}, achieving $\mathcal{O}(1/k^2)$ last-iterate rates for smooth convex-concave games. This anchored framework has been extended to encompass co-hypomonotone settings \cite{lee2021fast}, past-extragradient methods \cite{tran2021halpern,cai2022baccelerated}, and general inclusions \cite{cai2022accelerated,tran2023extragradient}. The connections between Halpern and Nesterov acceleration have been further elucidated in recent works \cite{tran2022connection,partkryu2022}. Moving-anchor variants, which alleviate the drag of a fixed initial point, have also been proposed, achieving tighter rates and sequence convergence \cite{alcala2023moving,yuan2024symplectic,lu2024restarted}.

We refer the reader to \cite{tran2023sublinear,tran2024revisiting,tran2025accelerated} for further discussion and a recent survey of developments.

\noindent\textbf{Variance Reduction and Antithetic Sampling:}
In stochastic variants of \eqref{eq:eg} and related optimization algorithms \cite{JohnsonZhang2013a,DefazioBach2014,GowerSchmidt2020,chen2021communication,das2022faster}, variance reduction is critical for improving sample efficiency. Distinct from the widely adopted mechanism of control variates, antithetic variates \cite{HammersleyMorton1956a,KahnMarshall1953,RubinsteinKroese2017} offer a computationally inexpensive method for variance reduction by leveraging negative correlation \cite{HammersleyMauldon1956,Whitt1976}. While empirically effective in machine learning applications \cite{WuGoodman2019,RenZhao2019,DoucetWang2005}, theoretical quantifications of antithetic variance reduction often rely on relatively weak monotonicity assumptions \cite{RubinsteinKroese2017}, and it has remained under-explored in optimization. 

Recently \cite{HashemiLeeMakur2025,hashemi2025strong} have strengthened these guarantees by employing strongly isotonic assumptions and defining the antithetic index of a distribution, establishing connections to optimal transport \cite{AmbrosioBrueSemola2021,PanaretosZemel2019} and concentration inequalities \cite{BoucheronLugosiMassart2013}. \eqref{eq:rampage+} leverages antithetic sampling to remove the first-order variance along the exploration segment, isolating the stochastic penalty to the higher-order non-linearities, thereby permitting a significantly larger stability threshold than the single-sample mechanism employed by \eqref{eq:rampage}.
\section{Details of the Infinite-Width GAN Example in Section \ref{sec:motivation}}\label{sec:gan}
We use GANs in the NTK regime \cite{jacot2018neural}. We consider a Dirac-GAN architecture \cite{mescheder2018training} where the true data distribution is defined as a singular point mass at the angular origin $\theta^* = 0$ on the unit circle. The generator's parameter is constrained to the unit circle $\mathbb{S}^1$, parameterized by the angle $\theta \in [-\pi, \pi]$. Equivalently, we may write the valid inputs as $x_\theta=(\cos\theta,\sin\theta)\in\mathbb{S}^1$, with the target point $x_0=(1,0)$ corresponding to the angular coordinate $\theta^*=0$. For the discriminator $D$, we assume the infinite-width limit, leading to an isomorphism between the discriminator and a linear functional operating within an RKHS $\mathcal{H}$. This space is uniquely governed by the deterministic covariance kernel $k(\cdot, \cdot)$ induced by the selected activation function \cite{jacot2018neural}. To give a concrete example, when the discriminator utilizes ReLU as activation, i.e. $\phi(z) = \max(0, z)$, the kernel evaluates to
$k(\theta) = \frac{1}{\pi} \big( \sin |\theta| + (\pi - |\theta|) \cos |\theta| \big)$ \cite{arora2019fine,li2019enhanced,golikov2022neural}.

Here, by a slight abuse of notation, $k(\theta)$ denotes the rotationally invariant angular profile $k(0,\theta)=k(x_0,x_\theta)$ of the two-argument RKHS kernel.

To ensure a bounded supremum for the functional objective, we consider the following regularized min-max game
\begin{equation}
\min_{\theta \in [-\pi, \pi]} \max_{D \in \mathcal{H}} \left( D(0) - D(\theta) - \frac{1}{2}\|D\|_{\mathcal{H}}^2 \right).
\end{equation}
By the Riesz representation theorem and the reproducing property of the kernel, $D(x) = \langle D, k(\cdot, x) \rangle_{\mathcal{H}}$ \cite{wainwright2019high}. Substituting these inner product representations, for a fixed $\theta$, the inner objective can be written as
$$\left\langle D, k(\cdot,0)-k(\cdot,\theta)\right\rangle_{\mathcal H} -\frac{1}{2}\|D\|_{\mathcal H}^2.$$
Setting the derivative of the objective with respect to the discriminator functional to zero yields
\begin{equation}
D^*(\cdot) = k(\cdot, 0) - k(\cdot, \theta).
\end{equation}
Substituting the optimal discriminator back into the full regularized objective, the generator then seeks to minimize
\begin{equation}
V(\theta) = D^*(0) - D^*(\theta) - \frac{1}{2}\|D^*\|_{\mathcal{H}}^2
= \frac{1}{2}\|k(\cdot,0)-k(\cdot,\theta)\|_{\mathcal{H}}^2.
\end{equation}
Using the kernel reproducing identity $\langle k(\cdot,x),k(\cdot,y)\rangle_{\mathcal H}=k(x,y)$, this gives
\begin{equation}
V(\theta) = \frac{1}{2}\big(k(0,0)-2k(0,\theta)+k(\theta,\theta)\big)
= k(0,0)-k(0,\theta),
\end{equation}
where we used $k(0,0)=k(\theta,\theta)$.

To see why $k(0, 0) = k(\theta, \theta)$, recall by our NTK construction \cite{jacot2018neural},
\begin{equation}
k(x, y) = \mathbb{E}_{w \sim \mathcal{N}(0, I_d)} \left[ \phi(w^\top x) \phi(w^\top y) \right].
\end{equation}
where we slightly abuse the notation to represent the functional evaluated at $x$ and $y$ corresponding to two different specific angles. Since $w^\top x \sim \mathcal{N}(0, \|x\|^2)$ and $\|x\| = 1$ for all valid inputs, we have
$ k(x, x) = \mathbb{E}_{Z \sim \mathcal{N}(0, 1)} \left[ \phi(Z)^2 \right],$
which is constant and independent of $x$. In particular, $k(0, 0) = k(\theta, \theta)$ for two valid inputs $x,y$ such that $\|x\| = \|y\| =1$, with $x$ having angle $0$ and $y$ having angle $\theta$. Further, as we saw for our ReLU example above, $k(x, y)$ depends only on the inner product $\langle x, y \rangle$ (see e.g. \cite{arora2019fine,li2019enhanced,golikov2022neural}) such that the evaluation $k(0, \theta)$ reduces to the even periodic scalar function $k(|\theta|)$.

Differentiating $V(\theta) = k(0,0) - k(|\theta|)$ yields
\begin{equation}
F(\theta) = \nabla_\theta V(\theta) = -\nabla_\theta k(|\theta|).
\end{equation}
\section{Estimation Error of EG, RAMPAGE, and RAMPAGE+}\label{app:variance}
In this section, we compare \eqref{eq:eg}, \eqref{eq:rampage}, and \eqref{eq:rampage+} in terms of their estimation error. While \eqref{eq:eg} is biased and has zero variance, \eqref{eq:rampage} and \eqref{eq:rampage+} are unbiased but have some variance due to their randomized nature.

\begin{proposition} \label{prop:estimation_error}
Let $\mathcal{I}_{\gamma} = \int_0^1 F(\theta_t - 2\eta s F_t) ds$ define the exact continuous-time average. The leading-order components evaluate to
\begin{equation} \label{eq:error_comparison}
\begin{aligned}
&\text{EG:} \quad \text{Bias} = -\frac{1}{6}\eta^2 H_t[F_t, F_t], \quad \text{Var} = 0, \\
&\text{RAMPAGE:} \quad \text{Bias} = 0, \quad \text{Var} = \frac{1}{3}\eta^2 \|J_t F_t\|^2 - \frac{2}{3}\eta^3 \langle J_t F_t, H_t[F_t, F_t] \rangle + \frac{16}{45}\eta^4 \|H_t[F_t, F_t]\|^2,\\
&\text{RAMPAGE+:} \quad\text{Bias} = 0, \quad \text{Var} = \frac{1}{45}\eta^4 \|H_t[F_t, F_t]\|^2.
\end{aligned}
\end{equation}
\end{proposition}

\begin{proof}
The Taylor expansion of the exact average $\mathcal{I}_{\gamma}$ yields
\begin{equation}
\mathcal{I}_{\gamma} = F_t - \eta J_t F_t + \frac{2}{3}\eta^2 H_t[F_t, F_t] + \mathcal{O}(\eta^3).
\end{equation}
For the deterministic Extragradient update $F_{EG} = F(\theta_t - \eta F_t)$, the Taylor expansion is
\begin{equation}
F_{EG} = F_t - \eta J_t F_t + \frac{1}{2}\eta^2 H_t[F_t, F_t] + \mathcal{O}(\eta^3).
\end{equation}
Because $F_{EG}$ is a deterministic point evaluation, its variance is zero. The bias evaluates exactly to
\begin{equation}
F_{EG} - \mathcal{I}_{\gamma} = \left(\frac{1}{2} - \frac{2}{3}\right)\eta^2 H_t[F_t, F_t] + \mathcal{O}(\eta^3) = -\frac{1}{6}\eta^2 H_t[F_t, F_t] + \mathcal{O}(\eta^3).
\end{equation}

For the stochastic \eqref{eq:rampage} estimator $F_{\text{RAMPAGE}} = F(\theta_t - 2\eta u_t F_t)$ parameterized by $u_t \sim \mathrm{Unif}([0,1])$, the expectation integrates exactly to $\mathcal{I}_{\gamma}$, rendering the bias zero. The instantaneous sampling error expands to
\begin{equation} \label{eq:rampage_full_error}
F_{\text{RAMPAGE}} - \mathcal{I}_{\gamma} = \eta (1 - 2u_t) J_t F_t + 2\eta^2 \left( u_t^2 - \frac{1}{3} \right) H_t[F_t, F_t] + \mathcal{O}(\eta^3).
\end{equation}
To find the variance,
\begin{equation}
\begin{aligned}
\operatorname{Var}(F_{\text{RAMPAGE}}) &= \eta^2 \|J_t F_t\|^2 \mathbb{E}\big[ (1 - 2u_t)^2 \big] \\
&\quad + 4\eta^3 \langle J_t F_t, H_t[F_t, F_t] \rangle \mathbb{E}\left[ (1 - 2u_t)\left(u_t^2 - \frac{1}{3}\right) \right] \\
&\quad + 4\eta^4 \|H_t[F_t, F_t]\|^2 \mathbb{E}\left[ \left(u_t^2 - \frac{1}{3}\right)^2 \right] + \mathcal{O}(\eta^5).
\end{aligned}
\end{equation}
We evaluate the continuous polynomial moments over the uniform distribution
\begin{equation}
\begin{aligned}
\mathbb{E}\big[ (1 - 2u_t)^2 \big] &= \int_0^1 (1 - 4s + 4s^2) ds = \frac{1}{3}, \\
\mathbb{E}\left[ (1 - 2u_t)\left(u_t^2 - \frac{1}{3}\right) \right] &= \int_0^1 \left( -2s^3 + s^2 + \frac{2}{3}s - \frac{1}{3} \right) ds = -\frac{1}{6}, \\
\mathbb{E}\left[ \left(u_t^2 - \frac{1}{3}\right)^2 \right] &= \int_0^1 \left( s^4 - \frac{2}{3}s^2 + \frac{1}{9} \right) ds = \frac{4}{45}.
\end{aligned}
\end{equation}
Substituting these integrations
\begin{equation}
\operatorname{Var}(F_{\text{RAMPAGE}}) = \frac{1}{3}\eta^2 \|J_t F_t\|^2 - \frac{2}{3}\eta^3 \langle J_t F_t, H_t[F_t, F_t] \rangle + \frac{16}{45}\eta^4 \|H_t[F_t, F_t]\|^2 + \mathcal{O}(\eta^5).
\end{equation}

For the antithetic \eqref{eq:rampage+} estimator $\bar{F}_t = \frac{1}{2} F(\theta_t - 2\eta u_t F_t) + \frac{1}{2} F(\theta_t - 2\eta \tilde{u}_t F_t)$, the combined update remains unbiased. Because $u_t + \tilde{u}_t = 1$, the first-order Jacobian terms perfectly cancel. The residual sampling error is driven entirely by the second-order curvature
\begin{equation}
\bar{F}_t - \mathcal{I}_{\gamma} = \eta^2 \left( 2u_t^2 - 2u_t + 1 - \frac{2}{3} \right) H_t[F_t, F_t] + \mathcal{O}(\eta^3) = \eta^2 \left( 2u_t^2 - 2u_t + \frac{1}{3} \right) H_t[F_t, F_t] + \mathcal{O}(\eta^3).
\end{equation}
Squaring this residual and taking the expectation evaluates the variance of~\eqref{eq:rampage+}
\begin{equation}
\operatorname{Var}(\bar{F}_t) = \eta^4 \|H_t[F_t, F_t]\|^2 \int_0^1 \left( 2s^2 - 2s + \frac{1}{3} \right)^2 ds + \mathcal{O}(\eta^5).
\end{equation}
 integrating the polynomial reduces to $1/45$.
\end{proof}
\section{Scaled Integration: Conservative vs. Non-Conservative}\label{app:scale}
To establish the geometric role of the exploration scale in numerical integration, we parameterize the exact line integral with a scaling factor $c >0$. The scaled continuous average over the exploration segment is defined as
\begin{equation} \label{eq:generalized_integral}
\mathcal{I}_{\gamma}^{(c)} = \int_0^1 F\big(\theta_t - c\eta s F(\theta_t)\big) ds.
\end{equation}
The choice of the scaling factor $c$ reveals a fundamental difference between conservative and non-conservative vector fields. We analyze two canonical examples to demonstrate that while $c=1$ is optimal for conservative minimization, $c > 1$ is mandatory to prevent divergence in non-conservative landscapes.

\noindent\textbf{Conservative Case:} First, consider a conservative vector field induced by a quadratic minimization problem $f(\theta) = \frac{1}{2} \theta^T H \theta$, where $H = H^T \succeq 0$ possesses a spectrum of real, non-negative eigenvalues $\lambda_i \in [0, L]$. The corresponding linear gradient flow is $F(\theta) = H\theta$. Integrating this field yields
\begin{equation}
\mathcal{I}_{\gamma}^{(c)} = \int_0^1 H(\theta_t - c\eta s H \theta_t) ds = H\theta_t - \frac{c\eta}{2}H^2\theta_t.
\end{equation}
The discrete update mapping utilizing this scaled integral takes the form
\begin{equation}
\theta_{t+1} = \theta_t - \eta \mathcal{I}_{\gamma}^{(c)} = \left( I - \eta H + \frac{c\eta^2}{2}H^2 \right) \theta_t.
\end{equation}
Because $H$ is symmetric and positive semi-definite, stability is entirely governed by the scalar spectral mapping of the individual eigenvalues $\lambda$
\begin{equation}
P_c(\lambda) = 1 - \eta \lambda + \frac{c\eta^2}{2}\lambda^2.
\end{equation}
For the discrete trajectory to converge, the spectral radius must not exceed unity, necessitating $-1\le P_c(\lambda) \le 1$. Evaluating this inequality yields
\begin{equation}
-1 \le 1 - \eta \lambda + \frac{c\eta^2}{2}\lambda^2 \le 1 \Longleftarrow \eta \le \frac{2}{c\lambda}, \quad \text{and} \quad c \ge 1/4.
\end{equation}
Since the continuous solution to the linear dynamical system $\dot{\theta} = -H\theta$ expands as $\theta(\eta) = \big(I - \eta H + \frac{\eta^2}{2}H^2 + \mathcal{O}(\eta^3)\big)\theta_t$, the unscaled integration with $c=1$ is theoretically optimal, guaranteeing stability for step sizes $\eta \le 2/L$. 

\noindent\textbf{Non-Conservative Case:} Conversely, we analyze a non-conservative case defined by an unconstrained bilinear min-max game. The field is purely skew-symmetric: $F(\theta) = S\theta$, with $S = -S^T$. Integrating the scaled field evaluates to $\mathcal{I}_{\gamma}^{(c)} = S\theta_t - \frac{c\eta}{2} S^2 \theta_t$. The corresponding discrete update matrix becomes $I - \eta S + \frac{c\eta^2}{2} S^2$. To evaluate absolute stability, we compute the physical energy expansion operator via the squared spectral norm
\begin{equation}
\begin{aligned}
\left\| I - \eta S + \frac{c\eta^2}{2} S^2 \right\|^2 &= \left( I + \eta S + \frac{c\eta^2}{2} S^2 \right)\left( I - \eta S + \frac{c\eta^2}{2} S^2 \right) \\
&= I + (c - 1)\eta^2 S^2 + \frac{c^2\eta^4}{4} S^4.
\end{aligned}
\end{equation}
Because the eigenvalues of a real skew-symmetric matrix are purely imaginary, $S^2$ is strictly negative semi-definite, while $S^4$ possesses strictly positive real eigenvalues. 

Evaluating the squared spectral norm of the discrete update matrix
\begin{equation}
P_c(\lambda) = 1 - (c - 1)\eta^2 \lambda^2 + \frac{c^2\eta^4}{4} \lambda^4.
\end{equation}
Absolute stability dictates that $P_c(\lambda) \le 1$ for all $|\lambda| \le L$. Evaluating this inequality leads to the admissible step size region
\begin{equation}
1 - (c - 1)\eta^2 \lambda^2 + \frac{c^2\eta^4}{4} \lambda^4 \le 1 \Longleftarrow \frac{c^2\eta^2}{4} \lambda^2 \le c - 1.
\end{equation}
This bound establishes the maximum permissible step size for any $c > 1$ to guarantee convergence
\begin{equation} \label{eq:scale_step_bound}
\eta \le \frac{2\sqrt{c - 1}}{c L}.
\end{equation}
To achieve maximum convergence, we optimize the scaling factor by maximizing the step size boundary $g(c) = \frac{2\sqrt{c - 1}}{c}$ over the domain $c > 1$. Evaluating the first derivative with respect to $c$ yields
\begin{equation}
\frac{d}{dc} g(c) = \frac{c(c-1)^{-1/2} - 2\sqrt{c-1}}{c^2} = \frac{2 - c}{c^2\sqrt{c-1}}.
\end{equation}
The critical point (which is the global maximum) occurs at $c = 2$.
\section{RAMPAGE and Its Analyses}\label{sec:rampage}
In this section, we analyze \eqref{eq:rampage} in numerous settings. Recall at each iteration $t \ge 0$, given a base iterate $\theta_t \in \mathbb{R}^p$ and a step size $\eta > 0$, we draw a single uniform random variable $u_t \sim \mathrm{Unif}([0, 1])$. The iterative scheme is formally defined by
\begin{equation} 
\begin{aligned}
y_t &= \theta_t - 2\eta u_t F(\theta_t), \qquad
\theta_{t+1} = \theta_t - \eta F(y_t).
\end{aligned}
\end{equation}
\subsection{Root-Finding Problems}
We first consider \eqref{eq:rf} with an operator $F: \mathbb{R}^p \to \mathbb{R}^p$ that is both $L$-Lipschitz continuous and $\mu$-co-coercive. 

\begin{theorem} \label{thm:esrmp_cocoercive}
Let $F: \mathbb{R}^p \to \mathbb{R}^p$ be an $L$-Lipschitz continuous and $\mu$-co-coercive operator with $\mu > 0$. Assume $\mathrm{zer}(F) \neq \emptyset$. Let $\{\theta_t\}_{t \ge 0}$ denote the sequence generated by the \eqref{eq:rampage}. If the constant step size satisfies $0 < \eta < \frac{1}{\sqrt{2}L} $ with $\eta \leq 2\mu$, then for any $\theta^* \in \mathrm{zer}(F)$ and with $C_{RA} = \eta^2 \left( 1 - 2L^2\eta^2 \right) > 0$
\begin{equation}
\min_{0 \le l \le k} \mathbb{E}\left[\|F(\theta_l)\|^2\right] \le \frac{1}{k+1} \sum_{l=0}^k \mathbb{E}\left[\|F(\theta_l)\|^2\right] \le \frac{\|\theta_0 - \theta^*\|^2}{C_{RA} (k+1)}.
\end{equation}
\end{theorem}
The proof is available in Appendix \ref{sec:esrmp_cocoercive}. The theorem thus demonstrates the standard $\mathcal{O}(1/k)$ convergence for this setting.

Next, we turn our attention to \eqref{eq:rf} under co-hypomonotonicity and Lipschitzness assumptions.
\begin{theorem} \label{thm:es_rmp_convergence}
Let $F: \mathbb{R}^p \to \mathbb{R}^p$ be an $L$-Lipschitz continuous and $\rho$-co-hypomonotone operator with $\rho \ge 0$. Assume $\mathrm{zer}(F) \neq \emptyset$. Let $\{\theta_t\}_{t \ge 0}$ denote the sequence generated by \eqref{eq:rampage}. If the step size $\eta$ satisfies
\begin{equation}
C_{RA}^{\rho,\eta} := \frac{1}{2} - \frac{2\rho}{\eta} - \frac{4}{3}L^2 \Big( \eta^2 + 2\eta\rho + 2(\eta + 2\rho)^2 \Big) > 0,
\end{equation}
then for any $\theta^* \in \mathrm{zer}(F)$,
\begin{equation}
\min_{0 \le l \le k} \mathbb{E}\left[\|F(\theta_l)\|^2\right] \le \frac{1}{k+1} \sum_{l=0}^k \mathbb{E}\left[\|F(\theta_l)\|^2\right] \le \frac{\|\theta_0 - \theta^*\|^2}{C_{RA}^{\rho,\eta} (k+1)}.
\end{equation}
\end{theorem}
The proof is available in Appendix \ref{sec:es_rmp_convergence}. When $\rho=0$, the operator is monotone. In this case, with some calculations, we obtain $
C_{RA}^{0,\eta} = \frac{1}{2} - 4L^2\eta^2$ and the upper bound $\eta < \frac{\sqrt{2}}{4L}$. Thus, the bound is slightly more restrictive compared to our previous result, which relied on co-coercivity.

Finally, we extend the analysis to \eqref{eq:rf} with monotone and $\alpha$-symmetric $(L_0, L_1)$-Lipschitz operators.
\begin{theorem} \label{thm:rampage_alpha_monotone}
Let $F: \mathbb{R}^p \to \mathbb{R}^p$ be a monotone and $\alpha$-symmetric $(L_0, L_1)$-Lipschitz operator with $\alpha \in (0, 1)$. Assume $\mathrm{zer}(F) \neq \emptyset$. Let $\{\theta_t\}_{t \ge 0}$ denote the sequence generated by \eqref{eq:rampage}. If the step size is set to $\eta_t = \frac{\nu}{K_0 + C_\alpha \|F(\theta_t)\|^\alpha}$ with $C_\alpha = K_1 + 2^{\frac{\alpha}{1-\alpha}} K_2$ and an appropriate absolute constant $\nu > 0$ such that
\begin{equation}
\nu \le \min \left\{ \frac{C_\alpha}{8(C_\alpha + K_1)}, \ C_\alpha \left( \frac{1}{8(C_\alpha - K_1)} \right)^{1-\alpha} \right\},
\end{equation}
then for any $\theta^* \in \mathrm{zer}(F)$,
\begin{equation}
\min_{0 \le l \le k} \mathbb{E}\left[ \eta_l^2 \|F(\theta_l)\|^2 \right] \le \frac{1}{k+1} \sum_{l=0}^k \mathbb{E}\left[ \eta_l^2 \|F(\theta_l)\|^2 \right] \le \frac{4\|\theta_0 - \theta^*\|^2}{k+1}.
\end{equation}
Consequently, the subsequence of best iterates $l^\ast_k$ enjoys $\|F(\theta_{l^\ast_k})\| \xrightarrow{p} 0$.
\end{theorem}
The proof is available in Appendix \ref{sec:rampage_alpha_monotone}. Note that when $L_1 = 0$ and $L_0 = L$, the operator is $L$-Lipschitz in the usual sense. In this limiting case, $K_1 = K_2 = 0$ and $K_0 =L$, and the step size reduces to $\eta = \nu/L$. Repeating the preceding argument directly gives the sufficient condition $\eta_t\le \frac{1}{8L}$. Thus, compared to $\eta < \frac{\sqrt{2}}{4L}$ which we established for Lipschitz operators directly in Theorem \ref{thm:es_rmp_convergence}, the upper bound is slightly more restrictive.
\subsection{Monotone Variational Inequalities and SS-RAMPAGE}
Recall we proposed the following Symmetrically Scaled (SS) variant. At each iteration $t \ge 0$, given a base iterate $\theta_t \in \mathcal{X}$ and a step size $\eta > 0$, we draw a single uniform random variable $u_t \sim \mathrm{Unif}([0, 1])$ and execute
\begin{equation}
\begin{aligned}
y_t = \Pi_{\mathcal{X}} \big( \theta_t - 2\eta u_t F(\theta_t) \big), \quad
\theta_{t+1} = \Pi_{\mathcal{X}} \big( \theta_t - 2\eta u_t F(y_t) \big).
\end{aligned}
\end{equation}

\begin{theorem} \label{thm:cp_rmp_monotone}
Let $F: \mathbb{R}^p \to \mathbb{R}^p$ be an $L$-Lipschitz continuous and monotone operator. Assume the Variational Inequality admits a solution $\theta^* \in \mathcal{X}$. Let $\{\theta_t\}_{t \ge 0}$ denote the sequence generated by \eqref{eq:ss-rampage}. If the step size satisfies $0 < \eta < \frac{1}{2L}$, then with $C_{SS} = 1 - 4\eta^2 L^2 > 0$
\begin{equation}
\min_{0 \le l \le k} \mathbb{E}\left[\|\theta_l - y_l\|^2\right] \le \frac{1}{k+1} \sum_{l=0}^k \mathbb{E}\left[\|\theta_l - y_l\|^2\right] \le \frac{\|\theta_0 - \theta^*\|^2}{C_{SS} (k+1)}.
\end{equation}
\end{theorem}
The proof is available in Appendix \ref{sec:cp_rmp_monotone}. We further state the following result on the expected restricted gap, whose proof is available in Appendix \ref{sec:cp_rmp_gap}.
\begin{corollary} \label{cor:cp_rmp_gap}
Let the conditions of Theorem \ref{thm:cp_rmp_monotone} hold. Assume that the sequence $\{\theta_l\}_{l=0}^k$ and the operator evaluations $\{F(\theta_l)\}_{l=0}^k$ are bounded almost surely such that $\sup_{l \ge 0, v \in \mathcal{B}} \|\theta_l - v\| \le D_{\mathcal{B}}$ and $\sup_{l \ge 0} \|F(\theta_l)\| \le G_{\mathcal{B}}$. Let $l^* = \mathrm{argmin}_{0 \le l \le k} \mathbb{E}[\|\theta_l - y_l\|^2]$ denote the index of the expected best iterate generated by \eqref{eq:ss-rampage} up to iteration $k$. Then, evaluating Lemma \ref{lem:gap} with the exact projection step size $2\eta u_{l^*}$, where $u_l\sim \mathrm{Unif}([\delta, 1])$ for some $\delta\in(0,1]$, yields
\begin{equation}
\mathbb{E}\left[\operatorname{Gap}_{\mathcal{B}}(\theta_{l^*})\right] \le \mathbb{E}\left[ \left( G_{\mathcal{B}} + \frac{D_{\mathcal{B}}}{2\eta u_{l^*}} \right) \frac{\|\theta_0 - \theta^*\|}{\sqrt{\big(1 - 4\eta^2 L^2\big)(k+1)}} \right].
\end{equation}
\end{corollary}
The result thus demonstrates the standard $\mathcal{O}(1/\sqrt{k})$ convergence for this setting.

\subsection{Application to Min-Max Games}
We now apply the \eqref{eq:rampage} to min-max optimization, specifically the unconstrained smooth convex-concave game. We will consider both deterministic and stochastic game settings.
\subsubsection{Deterministic Games}
The theorem below demonstrates the standard $\mathcal{O}(1/k)$ convergence for this setting.
\begin{theorem} \label{thm:ergodic_esrmp}
Let $F$ be the $L$-Lipschitz continuous monotone operator associated with the smooth convex-concave function $f$. Let $\{\theta_t\}_{t=0}^k$ and $\{y_t\}_{t=0}^k$ be the sequences generated by \eqref{eq:rampage}. Define the ergodic average over $k$ iterations as $\bar{y}_k = (\bar{x}_k, \bar{z}_k) = \frac{1}{k+1} \sum_{t=0}^k y_t$. If the step size satisfies $0 < \eta \le \frac{1}{2L}$, then for any reference point $\theta = (x, z)$,
\begin{equation}
\mathbb{E} \left[ f(\bar{x}_k, z) - f(x, \bar{z}_k) \right] \le \frac{\|\theta_0 - \theta\|^2}{2\eta (k+1)}.
\end{equation}
\end{theorem}
The proof is available in Appendix \ref{sec:ergodic_esrmp}. Note that the result shows ergodic convergence to an approximate Nash equilibrium.
\subsubsection{Stochastic Convex-Concave Games}
We now extend our analysis to stochastic games under Assumption \ref{ass:sfo} and state a result on the convergence properties of \eqref{eq:rampage}. The proof is available in Appendix \ref{sec:sfo_ergodic_rampage}. 
\begin{theorem} \label{thm:sfo_ergodic_rampage}
Let $F$ be the $L$-Lipschitz continuous monotone operator associated with the smooth convex-concave function $f$. Let $\{\theta_t\}_{t=0}^k$ and $\{y_t\}_{t=0}^k$ be the sequences generated by the SFO \eqref{eq:rampage} method. Define the ergodic average over $k$ iterations as $\bar{y}_k = (\bar{x}_k, \bar{z}_k) = \frac{1}{k+1} \sum_{t=0}^k y_t$. If the constant step size satisfies $0 < \eta \le \frac{1}{L \sqrt{2 C_R}}$, where $C_R = 4.4$, then for any reference point $\theta = (x, z)$,
\begin{equation}
\mathbb{E} \left[ f(\bar{x}_k, z) - f(x, \bar{z}_k) \right] \le \frac{\|\theta_0 - \theta\|^2}{2\eta (k+1)} + \frac{\eta \big(2 + \eta^2 L^2 C_R\big)}{2} \sigma^2.
\end{equation}
\end{theorem}
\section{Proofs of RAMPAGE+'s Results Stated in \ref{sec:rampage+}}
\subsection{Proof of Theorem \ref{thm:esramp_cocoercive}}\label{sec:esramp_cocoercive}
\begin{proof}[Proof of Theorem \ref{thm:esramp_cocoercive}]
We track $\|\theta_{t+1} - \theta^*\|^2$. Substituting the primary \eqref{eq:rampage+} update rule yields
\begin{equation} \label{eq:esramp_base_expansion}
\|\theta_{t+1} - \theta^*\|^2 = \|\theta_t - \theta^*\|^2 - 2\eta \langle \bar{F}_t, \theta_t - \theta^* \rangle + \eta^2 \|\bar{F}_t\|^2.
\end{equation}
We decouple the inner product via $y_t$ and $\tilde{y}_t$
\begin{equation}
\begin{aligned}
-2\eta \langle \bar{F}_t, \theta_t - \theta^* \rangle &= -\eta \langle F(y_t), \theta_t - \theta^* \rangle - \eta \langle F(\tilde{y}_t), \theta_t - \theta^* \rangle \\
&= -\eta \langle F(y_t), y_t - \theta^* \rangle - \eta \langle F(\tilde{y}_t), \tilde{y}_t - \theta^* \rangle \\
&\quad - 2\eta^2 u_t \langle F(y_t), F(\theta_t) \rangle - 2\eta^2 \tilde{u}_t \langle F(\tilde{y}_t), F(\theta_t) \rangle.
\end{aligned}
\end{equation}
Applying the polarization identity, $-2\langle a, b \rangle = \|a-b\|^2 - \|a\|^2 - \|b\|^2$, results in
\begin{equation}
\begin{aligned}
- 2\eta^2 u_t \langle F(y_t), F(\theta_t) \rangle &= -\eta^2 u_t \|F(y_t)\|^2 - \eta^2 u_t \|F(\theta_t)\|^2 + \eta^2 u_t \|F(y_t) - F(\theta_t)\|^2, \\
- 2\eta^2 \tilde{u}_t \langle F(\tilde{y}_t), F(\theta_t) \rangle &= -\eta^2 \tilde{u}_t \|F(\tilde{y}_t)\|^2 - \eta^2 \tilde{u}_t \|F(\theta_t)\|^2 + \eta^2 \tilde{u}_t \|F(\tilde{y}_t) - F(\theta_t)\|^2.
\end{aligned}
\end{equation}
Since $u_t + \tilde{u}_t = 1$, summing these components yields $-\eta^2 \|F(\theta_t)\|^2$. Applying the parallelogram law $\|a-b\|^2+\|a+b\|^2 = 2 \|a\|^2 + 2 \|b\|^2$ to $\eta^2 \|\bar{F}_t\|^2$ yields
\begin{equation}
\eta^2 \left\| \frac{1}{2} F(y_t) + \frac{1}{2} F(\tilde{y}_t) \right\|^2 = \frac{\eta^2}{2} \|F(y_t)\|^2 + \frac{\eta^2}{2} \|F(\tilde{y}_t)\|^2 - \frac{\eta^2}{4} \|F(y_t) - F(\tilde{y}_t)\|^2.
\end{equation}
Substituting these decompositions into \eqref{eq:esramp_base_expansion} results in
\begin{equation} \label{eq:esramp_grouped}
\begin{aligned}
\|\theta_{t+1} - \theta^*\|^2 &= \|\theta_t - \theta^*\|^2 - \eta \langle F(y_t), y_t - \theta^* \rangle - \eta \langle F(\tilde{y}_t), \tilde{y}_t - \theta^* \rangle \\
&\quad + \eta^2 \left( \frac{1}{2} - u_t \right) \|F(y_t)\|^2 + \eta^2 \left( \frac{1}{2} - \tilde{u}_t \right) \|F(\tilde{y}_t)\|^2 \\
&\quad - \eta^2 \|F(\theta_t)\|^2 + V_t,
\end{aligned}
\end{equation}
where $V_t$ isolates the residual
\begin{equation}
V_t = \eta^2 u_t \|F(y_t) - F(\theta_t)\|^2 + \eta^2 \tilde{u}_t \|F(\tilde{y}_t) - F(\theta_t)\|^2 - \frac{\eta^2}{4} \|F(y_t) - F(\tilde{y}_t)\|^2.
\end{equation}
By $\mu$-co-coercivity, $-\eta \langle F(y), y - \theta^* \rangle \le -\eta\mu \|F(y)\|^2$. Thus, the coefficient for $\|F(y_t)\|^2$ is $\frac{\eta^2}{2}(1 - 2u_t) - \eta\mu$. Since $u_t \in [0, 1]$, $(1 - 2u_t) \le 1$. Constraining $\eta \le 2\mu$ ensures this coefficient and its counterpart for $\tilde{y}_t$ remain non-positive almost surely. Discarding these terms yields
\begin{equation} \label{eq:esramp_variance_isolation}
\|\theta_{t+1} - \theta^*\|^2 \le \|\theta_t - \theta^*\|^2 - \eta^2 \|F(\theta_t)\|^2 + V_t.
\end{equation}
Define $A = F(y_t) - F(\theta_t)$ and $B = F(\tilde{y}_t) - F(\theta_t)$. Using $\|A - B\|^2 \ge (\|A\| - \|B\|)^2$ results in
\begin{equation}
V_t \le \eta^2 u_t \|A\|^2 + \eta^2 \tilde{u}_t \|B\|^2 - \frac{\eta^2}{4} \big( \|A\| - \|B\| \big)^2.
\end{equation}
By $L$-Lipschitz continuity of $F$, $\|A\| \le 2\eta L u_t \|F(\theta_t)\| = a$ and $\|B\| \le 2\eta L \tilde{u}_t \|F(\theta_t)\| = b$.  Maximizing the quadratic form $f(X,Y) = u_t X^2 + \tilde{u}_t Y^2 - \frac{1}{4}(X-Y)^2$ over $X \in [0, a]$ and $Y \in [0, b]$ occurs at $(a, b)$ since the Hessian determinant is $4u_t\tilde{u}_t - 1 \le 0$. Thus,
\begin{equation}
V_t \le \eta^2 \left( u_t a^2 + \tilde{u}_t b^2 - \frac{1}{4}(a-b)^2 \right).
\end{equation}
Substituting $a$ and $b$ yields
\begin{equation}
V_t \le 4\eta^4 L^2 \|F(\theta_t)\|^2 \left( u_t^3 + \tilde{u}_t^3 - \frac{1}{4}(u_t - \tilde{u}_t)^2 \right).
\end{equation}
Using $\tilde{u}_t = 1 - u_t$, we have $u_t^3 + \tilde{u}_t^3 = 1 - 3u_t + 3u_t^2$ and $-\frac{1}{4}(u_t - \tilde{u}_t)^2 = -u_t^2 + u_t - \frac{1}{4}$. Summing these results in
\begin{equation}
V_t \le 4\eta^4 L^2 \|F(\theta_t)\|^2 \left( \frac{3}{4} - 2u_t + 2u_t^2 \right).
\end{equation}
Since $3/4 - 2u_t + 2u_t^2 \leq 3/4$ almost surely
\begin{equation} \label{eq:esramp_final_variance}
V_t \le 3 \eta^4 L^2 \|F(\theta_t)\|^2.
\end{equation}
Substituting \eqref{eq:esramp_final_variance} into \eqref{eq:esramp_variance_isolation} yields
\begin{equation}
\|\theta_{t+1} - \theta^*\|^2\le \|\theta_t - \theta^*\|^2 - \eta^2 \left( 1 - 3 \eta^2 L^2 \right) \|F(\theta_t)\|^2.
\end{equation}
The condition $\eta \le \frac{1}{\sqrt{3}L}$ ensures $C_{RA+} > 0$. Summing from $t=0$ to $k$, and evaluating the minimum completes the proof.
\end{proof}
\subsection{Proof of Theorem \ref{thm:es_ramp_general}}\label{sec:es_ramp_general}
\begin{proof}[Proof of Theorem \ref{thm:es_ramp_general}]
We track $\|\theta_{t+1} - \theta^*\|^2$. By the primary update rule $\theta_{t+1} = \theta_t - \eta \bar{F}_t$,
\begin{equation} \label{eq:esramp_base}
\|\theta_{t+1} - \theta^*\|^2 = \|\theta_t - \theta^*\|^2 - 2\eta \langle \bar{F}_t, \theta_t - \theta^* \rangle + \eta^2 \|\bar{F}_t\|^2.
\end{equation}
We decouple the inner product via $y_t$ and $\tilde{y}_t$
\begin{equation}
\begin{aligned}
-2\eta \langle \bar{F}_t, \theta_t - \theta^* \rangle &= -\eta \langle F(y_t), y_t - \theta^* \rangle - \eta \langle F(\tilde{y}_t), \tilde{y}_t - \theta^* \rangle \\
&\quad - \eta \langle F(y_t), \theta_t - y_t \rangle - \eta \langle F(\tilde{y}_t), \theta_t - \tilde{y}_t \rangle.
\end{aligned}
\end{equation}
Substituting $\theta_t - y_t = 2\eta u_t F(\theta_t)$ and $\theta_t - \tilde{y}_t = 2\eta \tilde{u}_t F(\theta_t)$, and applying the polarization identity $-2\langle a, b \rangle = \|a-b\|^2 - \|a\|^2 - \|b\|^2$ to the cross-terms yields
\begin{equation}
\begin{aligned}
-2\eta^2 u_t \langle F(y_t), F(\theta_t) \rangle &= -\eta^2 u_t \|F(y_t)\|^2 - \eta^2 u_t \|F(\theta_t)\|^2 + \eta^2 u_t \|F(y_t) - F(\theta_t)\|^2, \\
-2\eta^2 \tilde{u}_t \langle F(\tilde{y}_t), F(\theta_t) \rangle &= -\eta^2 \tilde{u}_t \|F(\tilde{y}_t)\|^2 - \eta^2 \tilde{u}_t \|F(\theta_t)\|^2 + \eta^2 \tilde{u}_t \|F(\tilde{y}_t) - F(\theta_t)\|^2.
\end{aligned}
\end{equation}
Since $u_t + \tilde{u}_t = 1$, summing these yields $-\eta^2 \|F(\theta_t)\|^2$. Applying the parallelogram law $\|a-b\|^2+\|a+b\|^2 = 2 \|a\|^2 + 2 \|b\|^2$ to $\eta^2 \|\bar{F}_t\|^2$ results in
\begin{equation}
\eta^2 \|\bar{F}_t\|^2 = \frac{\eta^2}{2}\|F(y_t)\|^2 + \frac{\eta^2}{2}\|F(\tilde{y}_t)\|^2 - \frac{\eta^2}{4}\|F(y_t) - F(\tilde{y}_t)\|^2.
\end{equation}
Define $A = F(y_t) - F(\theta_t)$ and $B = F(\tilde{y}_t) - F(\theta_t)$. Substituting these into \eqref{eq:esramp_base} results in
\begin{equation} \label{eq:esramp_grouped3}
\begin{aligned}
\|\theta_{t+1} - \theta^*\|^2 &= \|\theta_t - \theta^*\|^2 - \eta^2 \|F(\theta_t)\|^2 - \eta \langle F(y_t), y_t - \theta^* \rangle - \eta \langle F(\tilde{y}_t), \tilde{y}_t - \theta^* \rangle \\
&\quad + \eta^2 \left(\frac{1}{2} - u_t\right)\|F(y_t)\|^2 + \eta^2 \left(\frac{1}{2} - \tilde{u}_t\right)\|F(\tilde{y}_t)\|^2 \\
&\quad + \eta^2 u_t \|A\|^2 + \eta^2 \tilde{u}_t \|B\|^2 - \frac{\eta^2}{4} \|A - B\|^2.
\end{aligned}
\end{equation}
By $\rho$-co-hypomonotonicity of $F$, $-\eta \langle F(y), y - \theta^* \rangle \le \eta\rho \|F(y)\|^2$. Substituting this into \eqref{eq:esramp_grouped3} and defining $c_t = \frac{1}{2} - u_t$, such that $\frac{1}{2} - \tilde{u}_t = -c_t$, the coefficients for $\|F(y_t)\|^2$ and $\|F(\tilde{y}_t)\|^2$ become $\eta\rho + \eta^2 c_t$ and $\eta\rho - \eta^2 c_t$, respectively.

Expanding 
\begin{equation}
 \|F(y_t)\|^2 = \|F(\theta_t)\|^2 + 2\langle F(\theta_t), A \rangle + \|A\|^2
\end{equation}
and
\begin{equation}
 \|F(\tilde{y}_t)\|^2 = \|F(\theta_t)\|^2 + 2\langle F(\theta_t), B \rangle + \|B\|^2,
\end{equation}
the coefficient of $\|F(\theta_t)\|^2$ becomes 
\begin{equation}
 (\eta\rho + \eta^2 c_t) + (\eta\rho - \eta^2 c_t) = 2\eta\rho,
\end{equation}
The total coefficient for $\|F(\theta_t)\|^2$ is $-( \eta^2 - 2\eta\rho )$.
The cross-term is $2\eta \langle F(\theta_t), (\rho + \eta c_t)A + (\rho - \eta c_t)B \rangle$. Define $Y_t = \rho(A+B) + \eta c_t (A-B)$. Applying Young's inequality,
\begin{equation}
 2\eta \langle F(\theta_t), Y_t \rangle \le \frac{1}{4} \eta^2 \|F(\theta_t)\|^2 + 4 \|Y_t\|^2. 
\end{equation}
The remaining coefficient for $\|F(\theta_t)\|^2$ reduces to $-\frac{3}{4}\eta^2 + 2\eta\rho$. Since $|c_t| \le 1/2$, 
\begin{equation}
 \|Y_t\| \le (\rho + \frac{\eta}{2})(\|A\| + \|B\|).
\end{equation}

By $L$-Lipschitz continuity of $F$, $\|A\| \le 2L\eta u_t \|F(\theta_t)\|$ and $\|B\| \le 2L\eta \tilde{u}_t \|F(\theta_t)\|$. Since $u_t + \tilde{u}_t = 1$, $\|A\| + \|B\| \le 2L\eta \|F(\theta_t)\|$. Thus, $4\|Y_t\|^2 \le 16L^2\eta^2(\rho + \frac{\eta}{2})^2 \|F(\theta_t)\|^2$.

The quadratic terms sum to 
\begin{equation}
 V_Q = (\eta\rho + \eta^2 c_t)\|A\|^2 + (\eta\rho - \eta^2 c_t)\|B\|^2 + \eta^2 u_t \|A\|^2 + \eta^2 \tilde{u}_t \|B\|^2.
\end{equation}
Substituting $c_t = 1/2 - u_t$, the coefficients for both $\|A\|^2$ and $\|B\|^2$ evaluate to $\eta\rho + \frac{\eta^2}{2}$. Thus, $V_Q = \eta(\rho + \frac{\eta}{2})(\|A\|^2 + \|B\|^2)$. Using the Lipschitz bounds, 
\begin{equation}
 \|A\|^2 + \|B\|^2 \le 4L^2\eta^2 (u_t^2 + \tilde{u}_t^2) \|F(\theta_t)\|^2.
\end{equation}
Since $u_t^2 + \tilde{u}_t^2 \le 1$, we have $V_Q \le 4L^2\eta^3(\rho + \frac{\eta}{2}) \|F(\theta_t)\|^2$.

Substituting these bounds yields
\begin{equation}
\|\theta_{t+1} - \theta^*\|^2 \le \|\theta_t - \theta^*\|^2 - \eta^2 \left( \frac{3}{4} - \frac{2\rho}{\eta} - 16L^2\left(\rho + \frac{\eta}{2}\right)^2 - 4L^2\eta\left(\rho + \frac{\eta}{2}\right) \right) \|F(\theta_t)\|^2.
\end{equation}
The condition on $\eta$ ensures $C_{RA+}^{\rho,\eta} > 0$. Summing from $t=0$ to $k$, dividing by $(k+1)$, and evaluating the minimum completes the proof.
\end{proof}
\subsection{Proof of Theorem \ref{thm:rampage_plus_alpha_monotone}}\label{sec:rampage_plus_alpha_monotone}
\begin{proof}[Proof of Theorem \ref{thm:rampage_plus_alpha_monotone}]
We track $\|\theta_{t+1} - \theta^*\|^2$. By the primary update rule,
\begin{equation}
\|\theta_{t+1} - \theta^*\|^2 = \|\theta_t - \theta^*\|^2 - 2\eta_t \langle \bar{F}_t, \theta_t - \theta^* \rangle + \eta_t^2 \|\bar{F}_t\|^2.
\end{equation}
We decouple the inner product via the antithetic intermediate states $y_t$ and $\tilde{y}_t$
\begin{equation}
-2\eta_t \langle \bar{F}_t, \theta_t - \theta^* \rangle = -\eta_t \langle F(y_t), y_t - \theta^* \rangle - \eta_t \langle F(\tilde{y}_t), \tilde{y}_t - \theta^* \rangle - \eta_t \langle F(y_t), \theta_t - y_t \rangle - \eta_t \langle F(\tilde{y}_t), \theta_t - \tilde{y}_t \rangle.
\end{equation}
By the monotonicity of $F$, the first two terms are non-positive and safely discarded. Substituting the exact displacements $\theta_t - y_t = 2\eta_t u_t F(\theta_t)$ and $\theta_t - \tilde{y}_t = 2\eta_t \tilde{u}_t F(\theta_t)$ yields
\begin{equation}
-\eta_t \langle F(y_t), \theta_t - y_t \rangle - \eta_t \langle F(\tilde{y}_t), \theta_t - \tilde{y}_t \rangle = -2\eta_t^2 u_t \langle F(y_t), F(\theta_t) \rangle - 2\eta_t^2 \tilde{u}_t \langle F(\tilde{y}_t), F(\theta_t) \rangle.
\end{equation}
Applying the polarization identity, $-2\langle a, b \rangle = \|a-b\|^2 - \|a\|^2 - \|b\|^2$, yields
\begin{equation}
\begin{aligned}
-2\eta_t^2 u_t \langle F(y_t), F(\theta_t) \rangle &= \eta_t^2 u_t \|F(y_t) - F(\theta_t)\|^2 - \eta_t^2 u_t \|F(y_t)\|^2 - \eta_t^2 u_t \|F(\theta_t)\|^2, \\
-2\eta_t^2 \tilde{u}_t \langle F(\tilde{y}_t), F(\theta_t) \rangle &= \eta_t^2 \tilde{u}_t \|F(\tilde{y}_t) - F(\theta_t)\|^2 - \eta_t^2 \tilde{u}_t \|F(\tilde{y}_t)\|^2 - \eta_t^2 \tilde{u}_t \|F(\theta_t)\|^2.
\end{aligned}
\end{equation}
Because $u_t + \tilde{u}_t = 1$, we have $-\eta_t^2 \|F(\theta_t)\|^2$. Expanding the squared magnitude of the averaged field yields $\eta_t^2 \|\bar{F}_t\|^2 = \frac{\eta_t^2}{2} \|F(y_t)\|^2 + \frac{\eta_t^2}{2} \|F(\tilde{y}_t)\|^2 - \frac{\eta_t^2}{4} \|F(y_t) - F(\tilde{y}_t)\|^2$.

Defining the deviations $A_t = F(y_t) - F(\theta_t)$ and $B_t = F(\tilde{y}_t) - F(\theta_t)$ and the symmetric scalar $c_t = \frac{1}{2} - u_t$, the aggregated distance recursion becomes
\begin{equation}
\begin{aligned}
\|\theta_{t+1} - \theta^*\|^2 &\le \|\theta_t - \theta^*\|^2 - \eta_t^2 \|F(\theta_t)\|^2 + \eta_t^2 c_t \|F(y_t)\|^2 - \eta_t^2 c_t \|F(\tilde{y}_t)\|^2 \\
&\quad + \eta_t^2 u_t \|A_t\|^2 + \eta_t^2 \tilde{u}_t \|B_t\|^2 - \frac{\eta_t^2}{4} \|A_t - B_t\|^2.
\end{aligned}
\end{equation}
Expanding the intermediate norms via $\|F(y_t)\|^2 = \|F(\theta_t)\|^2 + 2\langle F(\theta_t), A_t \rangle + \|A_t\|^2$ cancels the residual $\|F(\theta_t)\|^2$ coefficients because $c_t - c_t = 0$. Substituting $c_t + u_t = 1/2$ simplifies the variance to
\begin{equation} \label{eq:rampage_plus_variance_isolated}
\|\theta_{t+1} - \theta^*\|^2 \le \|\theta_t - \theta^*\|^2 - \eta_t^2 \|F(\theta_t)\|^2 + 2\eta_t^2 c_t \langle F(\theta_t), A_t - B_t \rangle + \frac{\eta_t^2}{2}\|A_t\|^2 + \frac{\eta_t^2}{2}\|B_t\|^2 - \frac{\eta_t^2}{4} \|A_t - B_t\|^2.
\end{equation}
Applying Young's inequality reduces the cross-term: $2 c_t \langle F(\theta_t), A_t - B_t \rangle \le \frac{1}{2}\|F(\theta_t)\|^2 + 2 c_t^2 \|A_t - B_t\|^2$. Because $c_t \in [-1/2, 1/2]$, we bound $c_t^2 \le 1/4$. Using $\|A_t - B_t\|^2 \le 2\|A_t\|^2 + 2\|B_t\|^2$, we merge the terms
\begin{equation}
\eta_t^2 \left( 2c_t^2 - \frac{1}{4} \right) \|A_t - B_t\|^2 \le \frac{\eta_t^2}{4} \big( 2\|A_t\|^2 + 2\|B_t\|^2 \big) = \frac{\eta_t^2}{2}\|A_t\|^2 + \frac{\eta_t^2}{2}\|B_t\|^2.
\end{equation}
Substituting this into \eqref{eq:rampage_plus_variance_isolated} guarantees the strict deterministic baseline bound
\begin{equation}
\|\theta_{t+1} - \theta^*\|^2 \le \|\theta_t - \theta^*\|^2 - \frac{1}{2}\eta_t^2 \|F(\theta_t)\|^2 + \eta_t^2 \|A_t\|^2 + \eta_t^2 \|B_t\|^2.
\end{equation}

Taking the conditional expectation $\mathbb{E}_t[\cdot] = \mathbb{E}[\cdot \mid \mathcal{F}_t]$, the symmetric sampling ensures $\mathbb{E}_t[\|A_t\|^2] = \mathbb{E}_t[\|B_t\|^2]$. By the exact $\alpha$-symmetric property mapped from $\theta_t$
\begin{equation}
\|A_t\|^2 \le 4\eta_t^2 u_t^2 \left( K_0 + K_1 \|F(\theta_t)\|^\alpha + K_2 (2\eta_t u_t \|F(\theta_t)\|)^{\frac{\alpha}{1-\alpha}} \right)^2 \|F(\theta_t)\|^2.
\end{equation}
Integrating the continuous fractional moments over $u_t \sim \mathrm{Unif}([0,1])$ yields $\int_0^1 s^2 ds = \frac{1}{3}$, $\int_0^1 s^{2 + \frac{\alpha}{1-\alpha}} ds = \frac{1-\alpha}{3-2\alpha}$, and $\int_0^1 s^{2 + \frac{2\alpha}{1-\alpha}} ds = \frac{1-\alpha}{3-\alpha}$. Since $\alpha \in (0,1)$, these precisely decay from the leading coefficient $1/3$. Bounding the polynomial
\begin{equation}
\mathbb{E}_t\big[ 2\eta_t^2 \|A_t\|^2 \big] \le \frac{8}{3}\eta_t^4 \|F(\theta_t)\|^2 \left( A_t^\circ + B_t^\circ \right)^2,
\end{equation}
where $A_t^\circ = K_0 + K_1\|F(\theta_t)\|^\alpha$ and $B_t^\circ = K_2 (2\eta_t\|F(\theta_t)\|)^{\frac{\alpha}{1-\alpha}}$. To enforce monotonic descent, we require the variance multiplier $\frac{8}{3}\eta_t^2(A_t^\circ + B_t^\circ)^2 \le \frac{1}{4}$. Decoupling the condition $Q_t = \eta_t(A_t^\circ + B_t^\circ) \le \sqrt{\frac{3}{32}}$, we substitute $\eta_t = \frac{\nu}{K_0 + C_\alpha \|F(\theta_t)\|^\alpha}$
\begin{equation}
\eta_t K_0 \le \nu, \quad \eta_t K_1 \|F(\theta_t)\|^\alpha \le \nu \frac{K_1}{C_\alpha}, \quad \eta_t B_t^\circ \le K_2 2^{\frac{\alpha}{1-\alpha}} \left(\frac{\nu}{C_\alpha}\right)^{\frac{1}{1-\alpha}}.
\end{equation}
By constraining $\nu > 0$ such that $Q_t \le \nu \left( 1 + \frac{K_1}{C_\alpha} \right) + K_2 2^{\frac{\alpha}{1-\alpha}} \left(\frac{\nu}{C_\alpha}\right)^{\frac{1}{1-\alpha}} \le \sqrt{\frac{3}{32}}$, which amounts to
\begin{equation}
\nu \le \min \left\{ \frac{\sqrt{6} C_\alpha}{16(C_\alpha + K_1)}, \ C_\alpha \left( \frac{\sqrt{6}}{16(C_\alpha - K_1)} \right)^{1-\alpha} \right\},
\end{equation}
the variance penalty collapses completely, generating the single-step dissipative sequence
\begin{equation}
\mathbb{E}_t\left[\|\theta_{t+1} - \theta^*\|^2\right] \le \|\theta_t - \theta^*\|^2 - \frac{1}{4}\eta_t^2 \|F(\theta_t)\|^2.
\end{equation}
Taking the total expectation and summing telescopically from $t=0$ to $k$ yields
\begin{equation}
\frac{1}{4} \sum_{t=0}^k \mathbb{E}\left[\eta_t^2 \|F(\theta_t)\|^2\right] \le \|\theta_0 - \theta^*\|^2 - \mathbb{E}\left[\|\theta_{k+1} - \theta^*\|^2\right] \le \|\theta_0 - \theta^*\|^2.
\end{equation}
Dividing by $k+1$ leads to the $\mathcal{O}(1/k)$ rate. Finally, because the mapping $M(x) = \left( \frac{\nu x}{K_0 + C_\alpha x^\alpha} \right)^2$ is a continuous, strictly increasing function over $[0, \infty)$, by Markov inequality and the continuous mapping theorem \cite{mann1943stochastic}, the subsequence of best iterates $l^\ast_k$ enjoys $\|F(\theta_{l^\ast_k})\| \xrightarrow{p} 0$.
\end{proof}
\subsection{Proof of Theorem \ref{thm:cp_ramp_monotone}}\label{sec:cp_ramp_monotone}
\begin{proof}[Proof of Theorem \ref{thm:cp_ramp_monotone}]
We track $\|\theta_{t+1} - \theta^*\|^2$. By the non-expansiveness of the projection $\theta_{t+1} = \Pi_{\mathcal{X}}(\theta_t - \eta u_t F(y_t) - \eta \tilde{u}_t F(\tilde{y}_t))$ with respect to $\theta^* \in \mathcal{X}$,
\begin{equation} \label{eq:cp_ramp_distance_expansion}
\|\theta_{t+1} - \theta^*\|^2 \le \|\theta_t - \theta^*\|^2 - \|\theta_{t+1} - \theta_t\|^2 - 2\eta \langle u_t F(y_t) + \tilde{u}_t F(\tilde{y}_t), \theta_{t+1} - \theta^* \rangle.
\end{equation}
We decouple the inner product into its constituent antithetic components
\begin{equation}
-2\eta \langle u_t F(y_t) + \tilde{u}_t F(\tilde{y}_t), \theta_{t+1} - \theta^* \rangle = -2\eta u_t \langle F(y_t), \theta_{t+1} - \theta^* \rangle - 2\eta \tilde{u}_t \langle F(\tilde{y}_t), \theta_{t+1} - \theta^* \rangle.
\end{equation}
We evaluate the primary trajectory component via the intermediate state $y_t$
\begin{equation} \label{eq:cp_ramp_t1_split}
-2\eta u_t \langle F(y_t), \theta_{t+1} - \theta^* \rangle = -2\eta u_t \langle F(y_t), y_t - \theta^* \rangle - 2\eta u_t \langle F(y_t), \theta_{t+1} - y_t \rangle.
\end{equation}
By the monotonicity of $F$ and the equilibrium condition of $\theta^*$, the first term is strictly non-positive. We expand the remaining differential cross-term by incorporating the base vector field
\begin{equation}
-2\eta u_t \langle F(y_t), \theta_{t+1} - y_t \rangle = -2\eta u_t \langle F(\theta_t), \theta_{t+1} - y_t \rangle - 2\eta u_t \langle F(y_t) - F(\theta_t), \theta_{t+1} - y_t \rangle.
\end{equation}
Testing the projection $y_t = \Pi_{\mathcal{X}}(\theta_t - 2\eta u_t F(\theta_t))$ against the updated feasible state $\theta_{t+1} \in \mathcal{X}$ yields
\begin{equation}
\langle \theta_t - 2\eta u_t F(\theta_t) - y_t, \theta_{t+1} - y_t \rangle \le 0.
\end{equation}
Rearranging this inequality
\begin{equation}
-2\eta u_t \langle F(\theta_t), \theta_{t+1} - y_t \rangle \le \langle y_t - \theta_t, \theta_{t+1} - y_t \rangle.
\end{equation}
Applying the polarization identity $2\langle a-b, b-c \rangle = \|a-c\|^2 - \|a-b\|^2 - \|b-c\|^2$ results in
\begin{equation}
\langle y_t - \theta_t, \theta_{t+1} - y_t \rangle = \frac{1}{2}\|\theta_{t+1} - \theta_t\|^2 - \frac{1}{2}\|\theta_{t+1} - y_t\|^2 - \frac{1}{2}\|\theta_t - y_t\|^2.
\end{equation}
Applying Young's inequality to the differential cross-term yields
\begin{equation}
-2\eta u_t \langle F(y_t) - F(\theta_t), \theta_{t+1} - y_t \rangle \le \frac{1}{2}\|\theta_{t+1} - y_t\|^2 + 2\eta^2 u_t^2 \|F(y_t) - F(\theta_t)\|^2.
\end{equation}
Summing these bounds perfectly cancels the penalty $\|\theta_{t+1} - y_t\|^2$ to yield
\begin{equation}
-2\eta u_t \langle F(y_t), \theta_{t+1} - \theta^* \rangle \le \frac{1}{2}\|\theta_{t+1} - \theta_t\|^2 - \frac{1}{2}\|\theta_t - y_t\|^2 + 2\eta^2 u_t^2 \|F(y_t) - F(\theta_t)\|^2.
\end{equation}
Symmetrical geometric analysis of the antithetic trajectory $\tilde{y}_t$ establishes the bound
\begin{equation}
-2\eta \tilde{u}_t \langle F(\tilde{y}_t), \theta_{t+1} - \theta^* \rangle \le \frac{1}{2}\|\theta_{t+1} - \theta_t\|^2 - \frac{1}{2}\|\theta_t - \tilde{y}_t\|^2 + 2\eta^2 \tilde{u}_t^2 \|F(\tilde{y}_t) - F(\theta_t)\|^2.
\end{equation}
Substituting the aggregated components back into the primary distance recursion completely neutralizes the initial divergence metric $\|\theta_{t+1} - \theta_t\|^2$ extracted from the outer projection
\begin{equation} \label{eq:cp_ramp_aggregated}
\begin{aligned}
\|\theta_{t+1} - \theta^*\|^2 &\le \|\theta_t - \theta^*\|^2 - \frac{1}{2}\|\theta_t - y_t\|^2 - \frac{1}{2}\|\theta_t - \tilde{y}_t\|^2 \\
&\quad + 2\eta^2 u_t^2 \|F(y_t) - F(\theta_t)\|^2 + 2\eta^2 \tilde{u}_t^2 \|F(\tilde{y}_t) - F(\theta_t)\|^2.
\end{aligned}
\end{equation}
Applying the $L$-Lipschitz continuity of the vector field
\begin{equation}
\|\theta_{t+1} - \theta^*\|^2 \le \|\theta_t - \theta^*\|^2 - \frac{1}{2} \big( 1 - 4\eta^2 u_t^2 L^2 \big) \|\theta_t - y_t\|^2 - \frac{1}{2} \big( 1 - 4\eta^2 \tilde{u}_t^2 L^2 \big) \|\theta_t - \tilde{y}_t\|^2.
\end{equation}
Since the uniform variables are strictly bounded such that $u_t \in [0, 1]$ and $\tilde{u}_t \in [0, 1]$, the respective descent multipliers satisfy $1 - 4\eta^2 u_t^2 L^2 \ge 1 - 4\eta^2 L^2$ and $1 - 4\eta^2 \tilde{u}_t^2 L^2 \ge 1 - 4\eta^2 L^2$. Enforcing the step size constraint $\eta \le \frac{1}{2L}$ guarantees that $C_{SS+} \ge 0$
\begin{equation}
\|\theta_{t+1} - \theta^*\|^2 \le \|\theta_t - \theta^*\|^2 - \frac{C_{SS+}}{2} \left( \|\theta_t - y_t\|^2 + \|\theta_t - \tilde{y}_t\|^2 \right).
\end{equation}
Taking the total expectation and telescoping the sequence from $t = 0$ to $k$ yields
\begin{equation}
\frac{C_{SS+}}{2} \sum_{t=0}^k \left( \|\theta_t - y_t\|^2 + \|\theta_t - \tilde{y}_t\|^2 \right) \le \|\theta_0 - \theta^*\|^2 - \|\theta_{k+1} - \theta^*\|^2 \le \|\theta_0 - \theta^*\|^2.
\end{equation}
Dividing by $(k+1)$ and noting the arithmetic mean is bounded below by the sequence minimum completes the proof.
\end{proof}
\subsection{Proof of Corollary \ref{cor:best_iterate_gap}}\label{sec:best_iterate_gap}
\begin{proof}[Proof of Corollary \ref{cor:best_iterate_gap}]
Let $\hat{y}_{l^*} \in \{y_{l^*}, \tilde{y}_{l^*}\}$ denote the intermediate exploration state generated by the maximum scaling parameter $\hat{u}_{l^*} = \max(u_{l^*}, \tilde{u}_{l^*})$. The corresponding projection is $\hat{y}_{l^*} = \Pi_{\mathcal{X}}(\theta_{l^*} - 2\eta \hat{u}_{l^*} F(\theta_{l^*}))$. Evaluating Lemma \ref{lem:gap} at $r_{l^*} = \theta_{l^*} - \hat{y}_{l^*}$ and $2\eta \hat{u}_{l^*}$ yields
\begin{equation}
\operatorname{Gap}_{\mathcal{B}}(\theta_{l^*}) \le \|\theta_{l^*} - \hat{y}_{l^*}\| \left( \|F(\theta_{l^*})\| + \frac{1}{2\eta \hat{u}_{l^*}} \sup_{v \in \mathcal{B}} \|\theta_{l^*} - v\| \right).
\end{equation}
Because the antithetic variables perfectly couple to satisfy $u_{l^*} + \tilde{u}_{l^*} = 1$, the maximum scaling parameter strictly satisfies $\hat{u}_{l^*} \ge 1/2$. Consequently, $\frac{1}{2\eta \hat{u}_{l^*}} \le \frac{1}{\eta}$. Substituting this and the absolute bounds $G_{\mathcal{B}}$ and $D_{\mathcal{B}}$ results in
\begin{equation}
\operatorname{Gap}_{\mathcal{B}}(\theta_{l^*}) \le \|\theta_{l^*} - \hat{y}_{l^*}\| \left( G_{\mathcal{B}} + \frac{D_{\mathcal{B}}}{\eta} \right).
\end{equation}
Since max is less than sum
\begin{equation}
\|\theta_{l^*} - \hat{y}_{l^*}\|^2 \le \|\theta_{l^*} - y_{l^*}\|^2 + \|\theta_{l^*} - \tilde{y}_{l^*}\|^2.
\end{equation}
Substituting the deterministic best-iterate bound from Theorem \ref{thm:cp_ramp_monotone} yields
\begin{equation}
\|\theta_{l^*} - y_{l^*}\|^2 + \|\theta_{l^*} - \tilde{y}_{l^*}\|^2 \le \frac{2\|\theta_0 - \theta^*\|^2}{\big(1 - 4\eta^2 L^2\big)(k+1)}.
\end{equation}
Taking the square root and substituting the derived bound into the gap inequality completes the proof.
\end{proof}
\subsection{Proof of Theorem \ref{thm:ergodic_esramp}}\label{sec:ergodic_esramp}
\begin{proof}[Proof of Theorem \ref{thm:ergodic_esramp}]
We track $\|\theta_{t+1} - \theta\|^2$. Substituting the primary update rule $\theta_{t+1} = \theta_t - \eta \bar{F}_t$ yields
\begin{equation} \label{eq:ramp_distance_expansion}
\|\theta_{t+1} - \theta\|^2 = \|\theta_t - \theta\|^2 - 2\eta \langle \bar{F}_t, \theta_t - \theta \rangle + \eta^2 \|\bar{F}_t\|^2.
\end{equation}
Define $\bar{y}_t = \frac{1}{2}(y_t + \tilde{y}_t)$. Substituting the algorithmic definitions and using $u_t + \tilde{u}_t = 1$, we have $\bar{y}_t = \theta_t - \eta F(\theta_t)$. We decouple the inner product via $\bar{y}_t$
\begin{equation} \label{eq:ramp_inner_product_split}
-2\eta \langle \bar{F}_t, \theta_t - \theta \rangle = -2\eta \langle \bar{F}_t, \bar{y}_t - \theta \rangle - 2\eta \langle \bar{F}_t, \theta_t - \bar{y}_t \rangle.
\end{equation}
Expanding $\bar{F}_t$ in the first term of \eqref{eq:ramp_inner_product_split} and applying the identity $\langle A+B, C+D \rangle = 2\langle A,C \rangle + 2\langle B,D \rangle - \langle A-B, C-D \rangle$ results in
\begin{equation} \label{eq:ramp_gap_identity}
-2\eta \langle \bar{F}_t, \bar{y}_t - \theta \rangle = -\eta \langle F(y_t), y_t - \theta \rangle - \eta \langle F(\tilde{y}_t), \tilde{y}_t - \theta \rangle + \frac{\eta}{2} \langle F(y_t) - F(\tilde{y}_t), y_t - \tilde{y}_t \rangle.
\end{equation}
Applying the convex-concave bound to $y_t = (x_{y,t}, z_{y,t})$ and $\tilde{y}_t = (x_{\tilde{y},t}, z_{\tilde{y},t})$ yields
\begin{equation} \label{eq:ramp_functional_substitution}
-\eta \langle F(y_t), y_t - \theta \rangle - \eta \langle F(\tilde{y}_t), \tilde{y}_t - \theta \rangle \le -\eta \big[ f(x_{y,t}, z) - f(x, z_{y,t}) \big] - \eta \big[ f(x_{\tilde{y},t}, z) - f(x, z_{\tilde{y},t}) \big].
\end{equation}
For the remaining terms, note that $\theta_t - \bar{y}_t = \eta F(\theta_t)$. Thus
\begin{equation}
-2\eta \langle \bar{F}_t, \theta_t - \bar{y}_t \rangle + \eta^2 \|\bar{F}_t\|^2 = -2\eta^2 \langle \bar{F}_t, F(\theta_t) \rangle + \eta^2 \|\bar{F}_t\|^2 = \eta^2 \|\bar{F}_t - F(\theta_t)\|^2 - \eta^2 \|F(\theta_t)\|^2.
\end{equation}
Bounding the discrepancy using the triangle inequality and $L$-Lipschitz continuity results in
\begin{equation}
\begin{aligned}
\eta^2 \|\bar{F}_t - F(\theta_t)\|^2 &\le \frac{\eta^2}{2}\|F(y_t) - F(\theta_t)\|^2 + \frac{\eta^2}{2}\|F(\tilde{y}_t) - F(\theta_t)\|^2 \\
&\le \frac{\eta^2 L^2}{2}\|y_t - \theta_t\|^2 + \frac{\eta^2 L^2}{2}\|\tilde{y}_t - \theta_t\|^2 \\
&= 2\eta^4 L^2 (u_t^2 + \tilde{u}_t^2) \|F(\theta_t)\|^2.
\end{aligned}
\end{equation}
Bounding the cross-term from \eqref{eq:ramp_gap_identity} via the Cauchy-Schwarz inequality and $L$-Lipschitz continuity yields
\begin{equation}
\frac{\eta}{2} \langle F(y_t) - F(\tilde{y}_t), y_t - \tilde{y}_t \rangle \le \frac{\eta L}{2} \|y_t - \tilde{y}_t\|^2 = \frac{\eta L}{2} (2\eta |u_t - \tilde{u}_t|)^2 \|F(\theta_t)\|^2 = 2\eta^3 L (2u_t - 1)^2 \|F(\theta_t)\|^2.
\end{equation}
Merging all terms modifying $\eta^2 \|F(\theta_t)\|^2$, the resulting coefficient is
\begin{equation}
C_{R} = -1 + 2\eta^2 L^2 \big(u_t^2 + (1-u_t)^2\big) + 2\eta L (2u_t - 1)^2.
\end{equation}
Since $u_t \in [0,1]$, we have $u_t^2 + (1-u_t)^2 \le 1$ and $(2u_t - 1)^2 \le 1$. Thus, $C_{R} \le -1 + 2\eta^2 L^2 + 2\eta L$. By the condition $\eta \le \frac{\sqrt{3}-1}{2L}$, $C_{R+} \le 0$. Dropping this non-positive term yields the descent inequality
\begin{equation} \label{eq:ramp_final_descent}
\eta \big[ f(x_{y,t}, z) - f(x, z_{y,t}) + f(x_{\tilde{y},t}, z) - f(x, z_{\tilde{y},t}) \big] \le \|\theta_t - \theta\|^2 - \|\theta_{t+1} - \theta\|^2.
\end{equation}
Summing from $t=0$ to $k$ telescopes
\begin{equation} \label{eq:ramp_telescopic_sum}
\sum_{t=0}^k \big[ f(x_{y,t}, z) - f(x, z_{y,t}) + f(x_{\tilde{y},t}, z) - f(x, z_{\tilde{y},t}) \big] \le \frac{\|\theta_0 - \theta\|^2 - \|\theta_{k+1} - \theta\|^2}{\eta} \le \frac{\|\theta_0 - \theta\|^2}{\eta}.
\end{equation}
Dividing by $2(k+1)$ and applying Jensen's inequality to the jointly convex function $f(\cdot, z) - f(x, \cdot)$ evaluated at $\hat{\theta}_k$ yields
\begin{equation} \label{eq:ramp_jensen_application}
f(\hat{x}_k, z) - f(x, \hat{z}_k) \le \frac{1}{2(k+1)} \sum_{t=0}^k \big[ f(x_{y,t}, z) - f(x, z_{y,t}) + f(x_{\tilde{y},t}, z) - f(x, z_{\tilde{y},t}) \big].
\end{equation}
Substituting \eqref{eq:ramp_telescopic_sum} into \eqref{eq:ramp_jensen_application} completes the proof.
\end{proof}
\subsection{Proof of Theorem \ref{thm:sfo_ergodic_rampage_plus}}\label{sec:sfo_ergodic_rampage_plus}
\begin{proof}[Proof of Theorem \ref{thm:sfo_ergodic_rampage_plus}]
We track the expected distance $\mathbb{E}[\|\theta_{t+1} - \theta\|^2]$. Substituting the primary stochastic update yields
\begin{equation}
\|\theta_{t+1} - \theta\|^2 = \|\theta_t - \theta\|^2 - 2\eta \langle \bar{G}_t, \theta_t - \theta \rangle + \eta^2 \|\bar{G}_t\|^2.
\end{equation}
Let $\mathcal{F}_{t, 1}$ denote the filtration conditioning on $\theta_t, u_t$, and $\xi_t$. Taking the conditional expectation with respect to the independent update samples $\zeta_t$ and $\tilde{\zeta}_t$, unbiasedness results in $\mathbb{E}[\bar{G}_t \mid \mathcal{F}_{t, 1}] = \frac{1}{2}(F(y_t) + F(\tilde{y}_t)) := \bar{F}_t$. The independence of $\zeta_t$ and $\tilde{\zeta}_t$ halves the variance of the averaged estimator
\begin{equation}
\mathbb{E}_{\zeta, \tilde{\zeta}}\left[\|\bar{G}_t\|^2 \mid \mathcal{F}_{t, 1}\right] \le \|\bar{F}_t\|^2 + \frac{\sigma^2}{2}.
\end{equation}
This yields the conditionally expected bound
\begin{equation} \label{eq:sfo_ramp_base_distance}
\mathbb{E}_{\zeta, \tilde{\zeta}}\left[\|\theta_{t+1} - \theta\|^2\right] \le \|\theta_t - \theta\|^2 - 2\eta \langle \bar{F}_t, \theta_t - \theta \rangle + \eta^2 \|\bar{F}_t\|^2 + \frac{\eta^2 \sigma^2}{2}.
\end{equation}
We decouple the inner product. Using $\theta_t - \theta = (\theta_t - y_t) + (y_t - \theta)$ and symmetrically for $\tilde{y}_t$, we obtain
\begin{equation}
\begin{aligned}
-2\eta \langle \bar{F}_t, \theta_t - \theta \rangle &= -\eta \langle F(y_t), y_t - \theta \rangle - \eta \langle F(\tilde{y}_t), \tilde{y}_t - \theta \rangle \\
&\quad - \eta \langle F(y_t), \theta_t - y_t \rangle - \eta \langle F(\tilde{y}_t), \theta_t - \tilde{y}_t \rangle.
\end{aligned}
\end{equation}
Let $\hat{F}_t = \hat{F}(\theta_t, \xi_t)$. Substituting $\theta_t - y_t = 2\eta u_t \hat{F}_t$ and $\theta_t - \tilde{y}_t = 2\eta \tilde{u}_t \hat{F}_t$ results in
\begin{equation}
-2\eta^2 u_t \langle F(y_t), \hat{F}_t \rangle - 2\eta^2 \tilde{u}_t \langle F(\tilde{y}_t), \hat{F}_t \rangle.
\end{equation}
Applying the polarization identity $-2\langle a, b \rangle = \|a-b\|^2 - \|a\|^2 - \|b\|^2$ and expanding $\eta^2\|\bar{F}_t\|^2 = \frac{\eta^2}{2}\|F(y_t)\|^2 + \frac{\eta^2}{2}\|F(\tilde{y}_t)\|^2 - \frac{\eta^2}{4}\|F(y_t) - F(\tilde{y}_t)\|^2$, the terms $-\eta^2 u_t \|\hat{F}_t\|^2$ and $-\eta^2 \tilde{u}_t \|\hat{F}_t\|^2$ sum to $-\eta^2 \|\hat{F}_t\|^2$. This leads to
\begin{equation} \label{eq:sfo_ramp_variance_block}
\begin{aligned}
R_t &= -\eta^2 \|\hat{F}_t\|^2 + \eta^2\left(\frac{1}{2} - u_t\right)\|F(y_t)\|^2 + \eta^2\left(\frac{1}{2} - \tilde{u}_t\right)\|F(\tilde{y}_t)\|^2 - \frac{\eta^2}{4}\|F(y_t) - F(\tilde{y}_t)\|^2 \\
&\quad + \eta^2 u_t \|F(y_t) - \hat{F}_t\|^2 + \eta^2 \tilde{u}_t \|F(\tilde{y}_t) - \hat{F}_t\|^2.
\end{aligned}
\end{equation}
Let $F_t = F(\theta_t)$, $A = F(y_t) - F_t$, $B = F(\tilde{y}_t) - F_t$, and $N_t = F_t - \hat{F}_t$. Expanding the negative squared norm results in
\begin{equation}
-\eta^2 \|\hat{F}_t\|^2 = -\eta^2 \|F_t - N_t\|^2 = -\eta^2 \|F_t\|^2 + 2\eta^2 \langle F_t, N_t \rangle - \eta^2 \|N_t\|^2.
\end{equation}
Bounding the discrepancy norms results in $\eta^2 u_t \|A + N_t\|^2 + \eta^2 \tilde{u}_t \|B + N_t\|^2 \le 2\eta^2 (u_t \|A\|^2 + \tilde{u}_t \|B\|^2) + 2\eta^2 \|N_t\|^2$. Taking the expectation over $\xi_t$, unbiasedness zeroes $2\eta^2 \langle F_t, N_t \rangle$. Let $\sigma_t^2 = \mathbb{E}[\|N_t\|^2] \le \sigma^2$. The noise terms combine to $\eta^2 \sigma_t^2$.

Expanding the intermediate norms and applying Young's inequality, we bound the variance using $\|A\|^2 \le 4\eta^2 L^2 u_t^2 \|\hat{F}_t\|^2$. Taking the expectation over $u_t$ results in
\begin{equation}
\mathbb{E}_{u_t} \left[ 4\eta^4 L^2 \big[ (2+u_t)u_t^2 + (3-u_t)(1-u_t)^2 \big] \|\hat{F}_t\|^2 \right] \le \eta^4 L^2 C_{R+} \big(\|F_t\|^2 + \sigma_t^2\big),
\end{equation}
where $C_{R+} = \frac{22}{3}$. Thus, the expected variance block reduces to
\begin{equation}
\mathbb{E}[R_t] \le -\eta^2 \left( \frac{3}{4} - \eta^2 L^2 C_{R+} \right) \|F_t\|^2 + \eta^2 \sigma_t^2 + \eta^4 L^2 C_{R+} \sigma_t^2.
\end{equation}
For $\eta \le \frac{\sqrt{3}}{2L\sqrt{C_{R+}}}$, the coefficient of $\|F_t\|^2$ is non-positive. Truncating this term and including the $\frac{\eta^2 \sigma^2}{2}$ variance from the primary update yields
\begin{equation}
\begin{aligned}
\eta \mathbb{E} \big[ f(x_{y,t}, z) - f(x, z_{y,t}) + f(x_{\tilde{y},t}, z) - f(x, z_{\tilde{y},t}) \big] \le \\
\mathbb{E} \left[ \|\theta_t - \theta\|^2 \right] - \mathbb{E} \left[ \|\theta_{t+1} - \theta\|^2 \right] + \eta^2 \left( \frac{3}{2} + \eta^2 L^2 C_{R+} \right) \sigma^2.
\end{aligned}
\end{equation}
Summing from $t=0$ to $k$ telescopes
\begin{equation}
\sum_{t=0}^k \mathbb{E} \big[ f(x_{y,t}, z) - f(x, z_{y,t}) + f(x_{\tilde{y},t}, z) - f(x, z_{\tilde{y},t}) \big] \le \frac{\|\theta_0 - \theta\|^2}{\eta} + (k+1)\eta \left( \frac{3}{2} + \eta^2 L^2 C_{R+} \right) \sigma^2.
\end{equation}
Dividing by $2(k+1)$ and applying Jensen's inequality to the jointly convex function evaluated at $\hat{\theta}_k$ yields
\begin{equation}
\mathbb{E} \left[ f(\hat{x}_k, z) - f(x, \hat{z}_k) \right] \le \frac{1}{2(k+1)} \sum_{t=0}^k \mathbb{E} \big[ f(x_{y,t}, z) - f(x, z_{y,t}) + f(x_{\tilde{y},t}, z) - f(x, z_{\tilde{y},t}) \big].
\end{equation}
Combining these inequalities completes the proof.
\end{proof}
\section{Proofs of RAMPAGE's Results Stated in Appendix \ref{sec:rampage}}
\subsection{Proof of Theorem \ref{thm:esrmp_cocoercive}}\label{sec:esrmp_cocoercive}
\begin{proof}[Proof of Theorem \ref{thm:esrmp_cocoercive}]
We track $\|\theta_{t+1} - \theta^*\|^2$. By substituting the primary update rule from \eqref{eq:rampage},
\begin{equation} \label{eq:esrmp_distance_expansion}
\|\theta_{t+1} - \theta^*\|^2 = \|\theta_t - \theta^*\|^2 - 2\eta \langle F(y_t), \theta_t - \theta^* \rangle + \eta^2 \|F(y_t)\|^2.
\end{equation}
To exploit the monotonicity characteristics of the field, we decouple the inner product
\begin{equation} \label{eq:esrmp_inner_product_split}
-2\eta \langle F(y_t), \theta_t - \theta^* \rangle = -2\eta \langle F(y_t), y_t - \theta^* \rangle - 2\eta \langle F(y_t), \theta_t - y_t \rangle.
\end{equation}
By $\mu$-co-coercivity of $F$ we have $\langle F(y_t) - F(\theta^*), y_t - \theta^* \rangle \ge \mu \|F(y_t) - F(\theta^*)\|^2$. Since $\theta^* \in \mathrm{zer}(F)$, this simplifies to $\langle F(y_t), y_t - \theta^* \rangle \ge \mu \|F(y_t)\|^2$. Applying this bound to the first term of \eqref{eq:esrmp_inner_product_split} yields
\begin{equation}
-2\eta \langle F(y_t), y_t - \theta^* \rangle \le -2\eta\mu \|F(y_t)\|^2.
\end{equation}
Next, we evaluate $\theta_t - y_t$. Substituting $y_t = \theta_t - 2\eta u_t F(\theta_t)$, $\theta_t - y_t = 2\eta u_t F(\theta_t)$ such that
\begin{equation}
-2\eta \langle F(y_t), \theta_t - y_t \rangle = -4\eta^2 u_t \langle F(y_t), F(\theta_t) \rangle.
\end{equation}
Using the polarization identity, $-2\langle a, b \rangle = \|a-b\|^2 - \|a\|^2 - \|b\|^2$
\begin{equation}
-4\eta^2 u_t \langle F(y_t), F(\theta_t) \rangle = 2\eta^2 u_t \|F(y_t) - F(\theta_t)\|^2 - 2\eta^2 u_t \|F(y_t)\|^2 - 2\eta^2 u_t \|F(\theta_t)\|^2.
\end{equation}
Using this in \eqref{eq:esrmp_distance_expansion} and simplifying we get
\begin{equation}
\begin{aligned}
\|\theta_{t+1} - \theta^*\|^2 &\le \|\theta_t - \theta^*\|^2 - \eta \big( 2\mu - \eta(1 - 2u_t) \big) \|F(y_t)\|^2 \\
&\quad - 2\eta^2 u_t \|F(\theta_t)\|^2 + 2\eta^2 u_t \|F(y_t) - F(\theta_t)\|^2.
\end{aligned}
\end{equation}
The coefficient $\|F(y_t)\|^2$ must remain non-negative across all possible random realizations. Since $u_t \in [0, 1]$, $(1 - 2u_t) \leq 1$. By enforcing the step size constraint $\eta \le 2\mu$, we guarantee that $2\mu - \eta(1 - 2u_t) \ge 0$ almost surely. Thus
\begin{equation} \label{eq:esrmp_discarded_fy}
\|\theta_{t+1} - \theta^*\|^2 \le \|\theta_t - \theta^*\|^2 - 2\eta^2 u_t \|F(\theta_t)\|^2 + 2\eta^2 u_t \|F(y_t) - F(\theta_t)\|^2.
\end{equation}
Using the $L$-Lipschitz continuity of the operator and the update rule
\begin{equation}
\|F(y_t) - F(\theta_t)\|^2 \le L^2 \|y_t - \theta_t\|^2 = 4\eta^2 L^2 u_t^2 \|F(\theta_t)\|^2.
\end{equation}
Substituting this into \eqref{eq:esrmp_discarded_fy} yields
\begin{equation} \label{eq:esrmp_pre_expectation}
\|\theta_{t+1} - \theta^*\|^2 \le \|\theta_t - \theta^*\|^2 - 2\eta^2 \big( u_t - 4\eta^2 L^2 u_t^3 \big) \|F(\theta_t)\|^2.
\end{equation}
Define the conditional expectation operator $\mathbb{E}_t[\cdot] = \mathbb{E}[\cdot \mid \mathcal{F}_t]$. Integrating the moments leads to $\mathbb{E}[u_t] = 1/2$ and $\mathbb{E}[u_t^3] = 1/4$. Thus,
\begin{equation}
\mathbb{E}_t \left[ 2\eta^2 \big( u_t - 4\eta^2 L^2 u_t^3 \big) \right] = 2\eta^2 \left( \frac{1}{2} - \frac{4\eta^2 L^2}{4} \right) = \eta^2 \left( 1 - 2\eta^2 L^2 \right).
\end{equation}
We thus obtain the expected single-step descent inequality
\begin{equation}
\mathbb{E}_t\left[\|\theta_{t+1} - \theta^*\|^2\right] \le \|\theta_t - \theta^*\|^2 - \eta^2 \left( 1 - 2\eta^2 L^2 \right) \|F(\theta_t)\|^2.
\end{equation}
Imposing the secondary constraint on the step size, $\eta < \frac{1}{\sqrt{2}L}$, ensures the positivity of $C_{RA} = \eta^2 (1 - 2\eta^2 L^2)$. Taking the total expectation and summing from $t = 0$ to $k$, the sum telescopes
\begin{equation}
C_{RA} \sum_{t=0}^k \mathbb{E}\left[\|F(\theta_t)\|^2\right] \le \|\theta_0 - \theta^*\|^2 - \mathbb{E}\left[\|\theta_{k+1} - \theta^*\|^2\right] \le \|\theta_0 - \theta^*\|^2.
\end{equation}
Dividing by $(k+1)$, rearranging, and noting that the minimum is bounded above by the arithmetic mean completes the proof.
\end{proof}
\subsection{Proof of Theorem \ref{thm:es_rmp_convergence}}\label{sec:es_rmp_convergence}
\begin{proof}[Proof of Theorem \ref{thm:es_rmp_convergence}]
We track $\|\theta_{t+1} - \theta^*\|^2$. By substituting the primary update rule,
\begin{equation} \label{eq:esrmp_initial_expansion}
\|\theta_{t+1} - \theta^*\|^2 = \|\theta_t - \theta^*\|^2 - 2\eta \langle F(y_t), \theta_t - \theta^* \rangle + \eta^2 \|F(y_t)\|^2.
\end{equation}
We decouple the inner product
\begin{equation}
- 2\eta \langle F(y_t), \theta_t - \theta^* \rangle = - 2\eta \langle F(y_t), y_t - \theta^* \rangle - 2\eta \langle F(y_t), \theta_t - y_t \rangle.
\end{equation}
Substituting $y_t = \theta_t - 2\eta u_t F(\theta_t)$, we have $\theta_t - y_t = 2\eta u_t F(\theta_t)$ such that
\begin{equation}
-2\eta \langle F(y_t), \theta_t - y_t \rangle = -4\eta^2 u_t \langle F(y_t), F(\theta_t) \rangle.
\end{equation}
Using the polarization identity, $-2\langle a, b \rangle = \|a-b\|^2 - \|a\|^2 - \|b\|^2$,
\begin{equation}
-4\eta^2 u_t \langle F(y_t), F(\theta_t) \rangle = 2\eta^2 u_t \|F(y_t) - F(\theta_t)\|^2 - 2\eta^2 u_t \|F(y_t)\|^2 - 2\eta^2 u_t \|F(\theta_t)\|^2.
\end{equation}
Define $A = F(y_t) - F(\theta_t)$. Using this in \eqref{eq:esrmp_initial_expansion} and simplifying, we get
\begin{equation} \label{eq:esrmp_grouped}
\|\theta_{t+1} - \theta^*\|^2 = \|\theta_t - \theta^*\|^2 - 2\eta \langle F(y_t), y_t - \theta^* \rangle + \eta^2(1 - 2u_t)\|F(y_t)\|^2 - 2\eta^2 u_t \|F(\theta_t)\|^2 + 2\eta^2 u_t \|A\|^2.
\end{equation}
We expand $\|F(y_t)\|^2$ via $A$
\begin{equation}
\|F(y_t)\|^2 = \|F(\theta_t) + A\|^2 = \|F(\theta_t)\|^2 + 2\langle F(\theta_t), A \rangle + \|A\|^2.
\end{equation}
Substituting this expansion into \eqref{eq:esrmp_grouped} yields
\begin{equation} \label{eq:esrmp_algebraic_collapse}
\|\theta_{t+1} - \theta^*\|^2 = \|\theta_t - \theta^*\|^2 - 2\eta \langle F(y_t), y_t - \theta^* \rangle + \eta^2(1 - 4u_t)\|F(\theta_t)\|^2 + 2\eta^2(1 - 2u_t)\langle F(\theta_t), A \rangle + \eta^2 \|A\|^2.
\end{equation}
By $\rho$-co-hypomonotonicity of $F$, we have $-2\eta \langle F(y_t), y_t - \theta^* \rangle \le 2\eta\rho \|F(y_t)\|^2$. Expanding $\|F(y_t)\|^2$ via $A$ and substituting the bound into \eqref{eq:esrmp_algebraic_collapse} results in
\begin{equation} \label{eq:esrmp_cohypo_applied}
\begin{aligned}
\|\theta_{t+1} - \theta^*\|^2 &\le \|\theta_t - \theta^*\|^2 + \big(\eta^2(1 - 4u_t) + 2\eta\rho\big)\|F(\theta_t)\|^2 \\
&\quad + 2\big(\eta^2(1 - 2u_t) + 2\eta\rho\big)\langle F(\theta_t), A \rangle + \big(\eta^2 + 2\eta\rho\big)\|A\|^2.
\end{aligned}
\end{equation}
To bound the cross-term, we apply Young's inequality. Let $\alpha_t = \eta^2(1 - 2u_t) + 2\eta\rho$. Then,
\begin{equation}
2\alpha_t \langle F(\theta_t), A \rangle \le \frac{\eta^2}{2} \|F(\theta_t)\|^2 + \frac{2\alpha_t^2}{\eta^2} \|A\|^2.
\end{equation}
Since $u_t \in [0, 1]$, $|\alpha_t| \le \eta^2 + 2\eta\rho$. Consequently, $\frac{2\alpha_t^2}{\eta^2} \le 2(\eta + 2\rho)^2$. Defining $\beta = \eta^2 + 2\eta\rho + 2(\eta + 2\rho)^2$, the total coefficient of $\|A\|^2$ is bounded by $\beta$.

Using the $L$-Lipschitz continuity of $F$,
\begin{equation}
\|A\|^2 = \|F(y_t) - F(\theta_t)\|^2 \le L^2\|y_t - \theta_t\|^2 = 4L^2 \eta^2 u_t^2 \|F(\theta_t)\|^2.
\end{equation}
Substituting this and the Young's inequality bound into \eqref{eq:esrmp_cohypo_applied} yields
\begin{equation}
\|\theta_{t+1} - \theta^*\|^2 \le \|\theta_t - \theta^*\|^2 + \left( \eta^2(1 - 4u_t) + 2\eta\rho + \frac{\eta^2}{2} + 4\beta L^2 \eta^2 u_t^2 \right) \|F(\theta_t)\|^2.
\end{equation}
Taking the conditional expectation $\mathbb{E}_t[\cdot] = \mathbb{E}[\cdot \mid \mathcal{F}_t]$ and noting $\mathbb{E}[1 - 4u_t] = -1$ and $\mathbb{E}[u_t^2] = 1/3$, we obtain the expected single-step descent inequality
\begin{equation}
\mathbb{E}_t\left[\|\theta_{t+1} - \theta^*\|^2\right] \le \|\theta_t - \theta^*\|^2 - \eta^2 \left( \frac{1}{2} - \frac{2\rho}{\eta} - \frac{4}{3}\beta L^2 \right) \|F(\theta_t)\|^2.
\end{equation}
Substituting $\beta$ yields
\begin{equation}
\mathbb{E}_t\left[\|\theta_{t+1} - \theta^*\|^2\right] \le \|\theta_t - \theta^*\|^2 - C_{RA}^{\rho,\eta} \|F(\theta_t)\|^2.
\end{equation}
The condition on $\eta$ ensures $C_{RA}^{\rho,\eta} > 0$. Taking the total expectation and summing from $t = 0$ to $k$, the sum telescopes
\begin{equation}
C_{RA}^{\rho,\eta} \sum_{t=0}^k \mathbb{E}\left[\|F(\theta_t)\|^2\right] \le \|\theta_0 - \theta^*\|^2 - \mathbb{E}\left[\|\theta_{k+1} - \theta^*\|^2\right] \le \|\theta_0 - \theta^*\|^2.
\end{equation}
Dividing by $(k+1)$, rearranging, and noting the arithmetic mean is bounded below by the minimum completes the proof.
\end{proof}
\subsection{Proof of Theorem \ref{thm:rampage_alpha_monotone}}\label{sec:rampage_alpha_monotone}
\begin{proof}[Proof of Theorem \ref{thm:rampage_alpha_monotone}]
We track $\|\theta_{t+1} - \theta^*\|^2$. Substituting the primary update rule,
\begin{equation}
\|\theta_{t+1} - \theta^*\|^2 = \|\theta_t - \theta^*\|^2 - 2\eta_t \langle F(y_t), \theta_t - \theta^* \rangle + \eta_t^2 \|F(y_t)\|^2.
\end{equation}
We decouple the inner product via $y_t$
\begin{equation}
-2\eta_t \langle F(y_t), \theta_t - \theta^* \rangle = -2\eta_t \langle F(y_t), y_t - \theta^* \rangle - 2\eta_t \langle F(y_t), \theta_t - y_t \rangle.
\end{equation}
Since $F$ is monotone and $\theta^* \in \mathrm{zer}(F)$, $\langle F(y_t), y_t - \theta^* \rangle \ge 0$. Thus, $-2\eta_t \langle F(y_t), y_t - \theta^* \rangle \le 0.$
Substituting $y_t = \theta_t - 2\eta_t u_t F(\theta_t)$ results in $\theta_t - y_t = 2\eta_t u_t F(\theta_t)$. Hence,
\begin{equation}
-2\eta_t \langle F(y_t), \theta_t - y_t \rangle = -4\eta_t^2 u_t \langle F(y_t), F(\theta_t) \rangle.
\end{equation}
Applying the polarization identity, $-2\langle a, b \rangle = \|a-b\|^2 - \|a\|^2 - \|b\|^2$, yields
\begin{equation}
-4\eta_t^2 u_t \langle F(y_t), F(\theta_t) \rangle = 2\eta_t^2 u_t \|F(y_t) - F(\theta_t)\|^2 - 2\eta_t^2 u_t \|F(y_t)\|^2 - 2\eta_t^2 u_t \|F(\theta_t)\|^2.
\end{equation}
Merging these terms yields
\begin{equation}
\|\theta_{t+1} - \theta^*\|^2 \le \|\theta_t - \theta^*\|^2 + \eta_t^2(1 - 2u_t)\|F(y_t)\|^2 - 2\eta_t^2 u_t \|F(\theta_t)\|^2 + 2\eta_t^2 u_t \|F(y_t) - F(\theta_t)\|^2.
\end{equation}
Define $A_t = F(y_t) - F(\theta_t)$. Expanding $\|F(y_t)\|^2 = \|F(\theta_t)\|^2 + 2\langle F(\theta_t), A_t \rangle + \|A_t\|^2$ and substituting results in
\begin{equation}
\|\theta_{t+1} - \theta^*\|^2 \le \|\theta_t - \theta^*\|^2 + \eta_t^2(1 - 4u_t)\|F(\theta_t)\|^2 + 2\eta_t^2(1 - 2u_t)\langle F(\theta_t), A_t \rangle + \eta_t^2 \|A_t\|^2.
\end{equation}
Applying Young's inequality, $2(1 - 2u_t)\langle F(\theta_t), A_t \rangle \le \frac{1}{2}\|F(\theta_t)\|^2 + 2(1 - 2u_t)^2\|A_t\|^2$. Merging terms yields
\begin{equation} \label{eq:rampage_alpha_variance_isolated}
\|\theta_{t+1} - \theta^*\|^2 \le \|\theta_t - \theta^*\|^2 + \eta_t^2\left(\frac{3}{2} - 4u_t\right)\|F(\theta_t)\|^2 + \eta_t^2\big(1 + 2(1 - 2u_t)^2\big)\|A_t\|^2.
\end{equation}
By the $\alpha$-symmetric property, we have
\begin{equation}
\|A_t\| \le \left( K_0 + K_1 \|F(\theta_t)\|^\alpha + K_2 \|y_t - \theta_t\|^{\frac{\alpha}{1-\alpha}} \right) \|y_t - \theta_t\|.
\end{equation}
Substituting $\|y_t - \theta_t\| = 2\eta_t u_t \|F(\theta_t)\|$ yields
\begin{equation} \label{eq:rampage_alpha_lipschitz_exact}
\|A_t\|^2 \le 4\eta_t^2 u_t^2 \left( K_0 + K_1 \|F(\theta_t)\|^\alpha + K_2 (2\eta_t u_t \|F(\theta_t)\|)^{\frac{\alpha}{1-\alpha}} \right)^2 \|F(\theta_t)\|^2.
\end{equation}
Taking the conditional expectation $\mathbb{E}_t[\cdot] = \mathbb{E}[\cdot \mid \mathcal{F}_t]$, we have $\mathbb{E}_t\left[\frac{3}{2} - 4u_t\right] = -\frac{1}{2}$. Bounding $1 + 2(1 - 2u_t)^2 \le 3$, we evaluate
\begin{equation}
\mathbb{E}_t\big[3 \|A_t\|^2\big] \le 12\eta_t^2 \|F(\theta_t)\|^2 \mathbb{E}_t\left[ u_t^2 \left( A_t^\circ + B_t^\circ u_t^{\frac{\alpha}{1-\alpha}} \right)^2 \right],
\end{equation}
where $A_t^\circ = K_0 + K_1\|F(\theta_t)\|^\alpha$ and $B_t^\circ = K_2 (2\eta_t\|F(\theta_t)\|)^{\frac{\alpha}{1-\alpha}}$. Evaluating the exact moments over the uniform distribution results in $\int_0^1 s^2 ds = \frac{1}{3}$, $\int_0^1 s^{2 + \frac{\alpha}{1-\alpha}} ds = \frac{1-\alpha}{3-2\alpha}$, and $\int_0^1 s^{2 + \frac{2\alpha}{1-\alpha}} ds = \frac{1-\alpha}{3-\alpha}$. Since $\frac{1}{3} > \frac{1-\alpha}{3-2\alpha} > \frac{1-\alpha}{3-\alpha}$ for $\alpha \in (0, 1)$, we bound
\begin{equation}
\mathbb{E}_t\big[3 \|A_t\|^2\big] \le 12\eta_t^2 \|F(\theta_t)\|^2 \left( \frac{(A_t^\circ)^2}{3} + 2A_t^\circ B_t^\circ \frac{1}{3} + \frac{(B_t^\circ)^2}{3} \right) = 4\eta_t^2 \|F(\theta_t)\|^2 (A_t^\circ + B_t^\circ)^2.
\end{equation}
To ensure expected descent, the variance multiplier $4\eta_t^2(A_t^\circ + B_t^\circ)^2$ must be bounded by the descent limit $\frac{1}{2}$. We require $Q_t = \eta_t(A_t^\circ + B_t^\circ) \le \frac{1}{4}$, then $4Q_t^2 = 4 \eta_t^2(A_t^\circ + B_t^\circ)^2 \leq \frac{1}{4}$. Using $\eta_t = \frac{\nu}{K_0 + C_\alpha \|F(\theta_t)\|^\alpha}$ with $C_\alpha = K_1 + 2^{\frac{\alpha}{1-\alpha}} K_2$, we bound the components of $Q_t$
\begin{equation} \label{eq:rampage_alpha_q1_q2}
\eta_t K_0 \le \frac{\nu K_0}{K_0} = \nu, \qquad \eta_t K_1 \|F(\theta_t)\|^\alpha \le \frac{\nu K_1 \|F(\theta_t)\|^\alpha}{C_\alpha \|F(\theta_t)\|^\alpha} = \nu \frac{K_1}{C_\alpha}.
\end{equation}
Using $\eta_t \le \frac{\nu}{C_\alpha \|F(\theta_t)\|^\alpha}$, we bound
\begin{equation} \label{eq:rampage_alpha_q3}
\eta_t B_t^\circ = K_2 2^{\frac{\alpha}{1-\alpha}} \eta_t^{\frac{1}{1-\alpha}} \|F(\theta_t)\|^{\frac{\alpha}{1-\alpha}} \le K_2 2^{\frac{\alpha}{1-\alpha}} \left( \frac{\nu}{C_\alpha \|F(\theta_t)\|^\alpha} \right)^{\frac{1}{1-\alpha}} \|F(\theta_t)\|^{\frac{\alpha}{1-\alpha}} = K_2 2^{\frac{\alpha}{1-\alpha}} \left(\frac{\nu}{C_\alpha}\right)^{\frac{1}{1-\alpha}}.
\end{equation}
Combining these bounds, we obtain
\begin{equation}
Q_t \le \nu \left( 1 + \frac{K_1}{C_\alpha} \right) + K_2 2^{\frac{\alpha}{1-\alpha}} \left(\frac{\nu}{C_\alpha}\right)^{\frac{1}{1-\alpha}}.
\end{equation}
By choosing $\nu > 0$ sufficiently small such that
\begin{equation}
\nu \le \min \left\{ \frac{C_\alpha}{8(C_\alpha + K_1)}, \ C_\alpha \left( \frac{1}{8(C_\alpha - K_1)} \right)^{1-\alpha} \right\},
\end{equation}
we ensure $Q_t \le \frac{1}{4}$. This bounds the variance penalty by $\frac{1}{4}\eta_t^2 \|F(\theta_t)\|^2$, yielding the expected descent inequality
\begin{equation}
\mathbb{E}_t\left[\|\theta_{t+1} - \theta^*\|^2\right] \le \|\theta_t - \theta^*\|^2 - \frac{1}{4}\eta_t^2 \|F(\theta_t)\|^2.
\end{equation}
Taking the total expectation yields
\begin{equation}
\mathbb{E}\left[\|\theta_{t+1} - \theta^*\|^2\right] \le \mathbb{E}\left[\|\theta_t - \theta^*\|^2\right] - \frac{1}{4}\mathbb{E}\left[\eta_t^2 \|F(\theta_t)\|^2\right].
\end{equation}
Summing from $t=0$ to $k$ telescopes
\begin{equation}
\frac{1}{4} \sum_{t=0}^k \mathbb{E}\left[\eta_t^2 \|F(\theta_t)\|^2\right] \le \|\theta_0 - \theta^*\|^2 - \mathbb{E}\left[\|\theta_{k+1} - \theta^*\|^2\right] \le \|\theta_0 - \theta^*\|^2.
\end{equation}
Dividing by $k+1$ and noting that the minimum is bounded by the average yields
\begin{equation}
\min_{0 \le l \le k} \mathbb{E}\left[ \eta_l^2 \|F(\theta_l)\|^2 \right] \le \frac{1}{k+1} \sum_{l=0}^k \mathbb{E}\left[ \eta_l^2 \|F(\theta_l)\|^2 \right] \le \frac{4\|\theta_0 - \theta^*\|^2}{k+1}.
\end{equation}
Define the function $M(x)$ for $x = \|F(\theta_l)\|$ using $\eta_l$
\begin{equation}
M(x) = \left( \frac{\nu x}{K_0 + C_\alpha x^\alpha} \right)^2.
\end{equation}
For $\alpha \in (0, 1)$, $M(x)$ is a continuous, strictly monotonically increasing function on $[0, \infty)$ with $M(x) = 0$ if and only if $x = 0$. Thus,  by Markov inequality and the continuous mapping theorem \cite{mann1943stochastic}, the subsequence of best iterates $l^\ast_k$ enjoys $\|F(\theta_{l^\ast_k})\| \xrightarrow{p} 0$.
\end{proof}
\subsection{Proof of Theorem \ref{thm:cp_rmp_monotone}}\label{sec:cp_rmp_monotone}
\begin{proof}[Proof of Theorem \ref{thm:cp_rmp_monotone}]
We track $\|\theta_{t+1} - \theta^*\|^2$. By the non-expansiveness of the projection $\theta_{t+1} = \Pi_{\mathcal{X}}(\theta_t - 2\eta u_t F(y_t))$ with respect to $\theta^* \in \mathcal{X}$,
\begin{equation} \label{eq:cp_rmp_distance_expansion}
\|\theta_{t+1} - \theta^*\|^2 \le \|\theta_t - \theta^*\|^2 - \|\theta_{t+1} - \theta_t\|^2 - 4\eta u_t \langle F(y_t), \theta_{t+1} - \theta^* \rangle.
\end{equation}
We decouple the inner product via the intermediate state $y_t$
\begin{equation} \label{eq:cp_rmp_inner_product_split}
-4\eta u_t \langle F(y_t), \theta_{t+1} - \theta^* \rangle = -4\eta u_t \langle F(y_t), y_t - \theta^* \rangle - 4\eta u_t \langle F(y_t), \theta_{t+1} - y_t \rangle.
\end{equation}
By the monotonicity of $F$ and the equilibrium condition of $\theta^*$, the first term is non-positive. We expand the remaining differential cross-term by adding and subtracting $F(\theta_t)$
\begin{equation}\label{proof:thm4:diff-crossterm}
-4\eta u_t \langle F(y_t), \theta_{t+1} - y_t \rangle = -4\eta u_t \langle F(\theta_t), \theta_{t+1} - y_t \rangle - 4\eta u_t \langle F(y_t) - F(\theta_t), \theta_{t+1} - y_t \rangle.
\end{equation}
Testing the projection $y_t = \Pi_{\mathcal{X}}(\theta_t - 2\eta u_t F(\theta_t))$ against the updated state $\theta_{t+1} \in \mathcal{X}$ yields
\begin{equation}
\langle \theta_t - 2\eta u_t F(\theta_t) - y_t, \theta_{t+1} - y_t \rangle \le 0.
\end{equation}
Rearranging this inequality
\begin{equation}
-4\eta u_t \langle F(\theta_t), \theta_{t+1} - y_t \rangle \le 2\langle y_t - \theta_t, \theta_{t+1} - y_t \rangle.
\end{equation}
Applying the polarization identity $2\langle a-b, b-c \rangle = \|a-c\|^2 - \|a-b\|^2 - \|b-c\|^2$ results in
\begin{equation}\label{proof:thm4:polarization}
-4\eta u_t \langle F(\theta_t), \theta_{t+1} - y_t \rangle \le \|\theta_{t+1} - \theta_t\|^2 - \|\theta_{t+1} - y_t\|^2 - \|\theta_t - y_t\|^2.
\end{equation}
Applying Young's inequality to the second term at RHS of equality~(\ref{proof:thm4:diff-crossterm})
\begin{equation}\label{proof:thm4:young-ineq}
-4\eta u_t \langle F(y_t) - F(\theta_t), \theta_{t+1} - y_t \rangle \le \|\theta_{t+1} - y_t\|^2 + 4\eta^2 u_t^2 \|F(y_t) - F(\theta_t)\|^2.
\end{equation}
Substituting the bounds in~(\ref{proof:thm4:young-ineq}) and~(\ref{proof:thm4:polarization}) into the primary distance recursion results in~(\ref{eq:cp_rmp_distance_expansion}) using~(\ref{proof:thm4:diff-crossterm})
\begin{equation}
\|\theta_{t+1} - \theta^*\|^2 \le \|\theta_t - \theta^*\|^2 - \|\theta_t - y_t\|^2 + 4\eta^2 u_t^2 \|F(y_t) - F(\theta_t)\|^2.
\end{equation}
Using the $L$-Lipschitz continuity of the operator
\begin{equation}
\|F(y_t) - F(\theta_t)\|^2 \le L^2 \|y_t - \theta_t\|^2.
\end{equation}
Substituting this into the distance recursion yields
\begin{equation}\label{proof:thm4:distance-recursion}
\|\theta_{t+1} - \theta^*\|^2 \le \|\theta_t - \theta^*\|^2 - \big( 1 - 4\eta^2 u_t^2 L^2 \big) \|\theta_t - y_t\|^2.
\end{equation}
Since $u_t \in [0, 1]$, the coefficient satisfies $1 - 4\eta^2 u_t^2 L^2 \ge 1 - 4\eta^2 L^2$. By enforcing the step size constraint $\eta < \frac{1}{2L}$, we guarantee that $C_{SS} = 1 - 4\eta^2 L^2 > 0$. Taking the conditional expectation operator $\mathbb{E}_t[\cdot] = \mathbb{E}[\cdot \mid \mathcal{F}_t]$ yields the expected single-step descent inequality
\begin{equation}\label{proof:thm4:single-step-descent}
\mathbb{E}_t\left[\|\theta_{t+1} - \theta^*\|^2\right] \le \|\theta_t - \theta^*\|^2 - C_{SS} \mathbb{E}_t\left[\|\theta_t - y_t\|^2\right].
\end{equation}
Taking the total expectation and summing from $t = 0$ to $k$, the sum telescopes
\begin{equation}
C_{SS} \sum_{t=0}^k \mathbb{E}\left[\|\theta_t - y_t\|^2\right] \le \|\theta_0 - \theta^*\|^2 - \mathbb{E}\left[\|\theta_{k+1} - \theta^*\|^2\right] \le \|\theta_0 - \theta^*\|^2.
\end{equation}
Dividing by $(k+1)$, rearranging, and noting the arithmetic mean is bounded below by the minimum completes the proof.
\end{proof}
\subsection{Proof of Corollary \ref{cor:cp_rmp_gap}}\label{sec:cp_rmp_gap}
\begin{proof}[Proof of Corollary \ref{cor:cp_rmp_gap}]
We first observe that the proof of Theorem~\ref{thm:cp_rmp_monotone} remains valid when $u_l\sim \mathrm{Unif}([\delta, 1])$, since the argument only relies on the upper bound of $u_l$ from~(\ref{proof:thm4:distance-recursion}) to~(\ref{proof:thm4:single-step-descent}). Therefore, we can apply Theorem~\ref{thm:cp_rmp_monotone} under this modified scheme. By the algorithmic definition, the intermediate state is generated by the projection $y_{l^*} = \Pi_{\mathcal{X}}(\theta_{l^*} - 2\eta u_{l^*} F(\theta_{l^*}))$. Evaluating Lemma \ref{lem:gap} utilizing the projection residual $r_{l^*} = \theta_{l^*} - y_{l^*}$ and the realized algorithmic step size $2\eta u_{l^*}$ yields
\begin{equation}
\operatorname{Gap}_{\mathcal{B}}(\theta_{l^*}) \le \|\theta_{l^*} - y_{l^*}\| \left( \|F(\theta_{l^*})\| + \frac{1}{2\eta u_{l^*}} \sup_{v \in \mathcal{B}} \|\theta_{l^*} - v\| \right).
\end{equation}
Substituting the almost sure bounds $\|F(\theta_{l^*})\| \le G_{\mathcal{B}}$ and $\sup_{v \in \mathcal{B}} \|\theta_{l^*} - v\| \le D_{\mathcal{B}}$ results in
\begin{equation}
\operatorname{Gap}_{\mathcal{B}}(\theta_{l^*}) \le \|\theta_{l^*} - y_{l^*}\| \left( G_{\mathcal{B}} + \frac{D_{\mathcal{B}}}{2\eta u_{l^*}} \right).
\end{equation}
Taking the expectation and applying Jensen's inequality bounding $\mathbb{E}[\|\theta_{l^*} - y_{l^*}\|] \le \sqrt{\mathbb{E}[\|\theta_{l^*} - y_{l^*}\|^2]}$ isolates the residual norm. Substituting the expected best-iterate bound from Theorem \ref{thm:cp_rmp_monotone} completes the proof.
\end{proof}
\subsection{Proof of Theorem \ref{thm:ergodic_esrmp}}\label{sec:ergodic_esrmp}
\begin{proof}[Proof of Theorem \ref{thm:ergodic_esrmp}]
We track $\|\theta_{t+1} - \theta\|^2$. Substituting the primary update rule $\theta_{t+1} = \theta_t - \eta F(y_t)$ yields
\begin{equation} \label{eq:distance_expansion}
\|\theta_{t+1} - \theta\|^2 = \|\theta_t - \theta\|^2 - 2\eta \langle F(y_t), \theta_t - \theta \rangle + \eta^2 \|F(y_t)\|^2.
\end{equation}
We decouple the inner product via $y_t$
\begin{equation} \label{eq:inner_product_split}
-2\eta \langle F(y_t), \theta_t - \theta \rangle = -2\eta \langle F(y_t), y_t - \theta \rangle - 2\eta \langle F(y_t), \theta_t - y_t \rangle.
\end{equation}
Substituting $y_t = \theta_t - 2\eta u_t F(\theta_t)$, we have $\theta_t - y_t = 2\eta u_t F(\theta_t)$. Thus,
\begin{equation} \label{eq:displacement_sub}
-2\eta \langle F(y_t), \theta_t - y_t \rangle = -4\eta^2 u_t \langle F(y_t), F(\theta_t) \rangle.
\end{equation}
Let $E_t$ denote the $\mathcal{O}(\eta^2)$ terms from \eqref{eq:distance_expansion}
\begin{equation}
E_t = \eta^2 \|F(y_t)\|^2 - 4\eta^2 u_t \langle F(y_t), F(\theta_t) \rangle.
\end{equation}
Adding and subtracting $4\eta^2 u_t^2 \|F(\theta_t)\|^2$ results in
\begin{equation} \label{eq:error_square_completion}
E_t = \eta^2 \|F(y_t) - 2u_t F(\theta_t)\|^2 - 4\eta^2 u_t^2 \|F(\theta_t)\|^2.
\end{equation}
Define $A_t = F(y_t) - F(\theta_t)$. Applying $\|a + b\|^2 \le 2\|a\|^2 + 2\|b\|^2$ on $\|A_t + (1-2u_t)F(\theta_t)\|^2$ yields
\begin{equation} \label{eq:error_geometric_bound}
\eta^2 \|F(y_t) - 2u_t F(\theta_t)\|^2 \le 2\eta^2 \|A_t\|^2 + 2\eta^2 (1-2u_t)^2 \|F(\theta_t)\|^2.
\end{equation}
By $L$-Lipschitz continuity of $F$,
\begin{equation} \label{eq:lipschitz_application}
2\eta^2 \|A_t\|^2 \le 2\eta^2 L^2 \|y_t - \theta_t\|^2 = 8\eta^4 L^2 u_t^2 \|F(\theta_t)\|^2.
\end{equation}
Combining these bounds yields
\begin{equation} \label{eq:error_total_unconditioned}
E_t \le \eta^2 \left( 8\eta^2 L^2 u_t^2 + 2(1-2u_t)^2 - 4u_t^2 \right) \|F(\theta_t)\|^2.
\end{equation}
Taking the conditional expectation $\mathbb{E}_t[\cdot] = \mathbb{E}[\cdot \mid \mathcal{F}_t]$ and noting $\mathbb{E}_t[u_t^2] = 1/3$ and $\mathbb{E}_t[(1-2u_t)^2] = 1/3$ results in
\begin{equation} \label{eq:error_expectation_collapsed}
\mathbb{E}_t[E_t] \le \eta^2 \left( \frac{8}{3}\eta^2 L^2 + \frac{2}{3} - \frac{4}{3} \right) \|F(\theta_t)\|^2 = \frac{2}{3}\eta^2 \left( 4\eta^2 L^2 - 1 \right) \|F(\theta_t)\|^2.
\end{equation}
For $\eta \le \frac{1}{2L}$, $4\eta^2 L^2 - 1 \le 0$. Applying this to \eqref{eq:distance_expansion} yields the expected descent inequality
\begin{equation} \label{eq:fundamental_descent_expectation}
2\eta \mathbb{E}_t \left[ \langle F(y_t), y_t - \theta \rangle \right] \le \|\theta_t - \theta\|^2 - \mathbb{E}_t \left[ \|\theta_{t+1} - \theta\|^2 \right].
\end{equation}
Applying \eqref{eq:convex_concave_bound} and taking the total expectation yields
\begin{equation}
2\eta \mathbb{E} \left[ f(x_t^m, z) - f(x, z_t^m) \right] \le \mathbb{E} \left[ \|\theta_t - \theta\|^2 \right] - \mathbb{E} \left[ \|\theta_{t+1} - \theta\|^2 \right].
\end{equation}
Summing from $t=0$ to $k$ telescopes
\begin{equation} \label{eq:telescopic_sum}
\sum_{t=0}^k \mathbb{E} \left[ f(x_t^m, z) - f(x, z_t^m) \right] \le \frac{\|\theta_0 - \theta\|^2 - \mathbb{E} \left[ \|\theta_{k+1} - \theta\|^2 \right]}{2\eta} \le \frac{\|\theta_0 - \theta\|^2}{2\eta}.
\end{equation}
Dividing by $(k+1)$ and applying Jensen's inequality to the jointly convex function $f(\cdot, z) - f(x, \cdot)$ evaluated at the ergodic average $\bar{y}_k$ completes the proof.
\end{proof}
\subsection{Proof of Theorem \ref{thm:sfo_ergodic_rampage}}\label{sec:sfo_ergodic_rampage}
\begin{proof}[Proof of Theorem \ref{thm:sfo_ergodic_rampage}]
We track $\mathbb{E}[\|\theta_{t+1} - \theta\|^2]$. Denote $F_t = F(\theta_t)$, $\hat{F}_t = \hat{F}(\theta_t, \xi_t)$, and $\hat{G}_t = \hat{F}(y_t, \zeta_t)$. Substituting the primary update rule yields
\begin{equation}
\|\theta_{t+1} - \theta\|^2 = \|\theta_t - \theta\|^2 - 2\eta \langle \hat{G}_t, \theta_t - \theta \rangle + \eta^2 \|\hat{G}_t\|^2.
\end{equation}
Let $\mathcal{F}_{t, 1}$ denote the filtration conditioning on $\theta_t, u_t$, and $\xi_t$. Taking the conditional expectation with respect to $\zeta_t$, unbiasedness ensures $\mathbb{E}_{\zeta_t}[\hat{G}_t] = F(y_t)$ and $\mathbb{E}_{\zeta_t}[\|\hat{G}_t\|^2] \le \|F(y_t)\|^2 + \sigma^2$. This yields
\begin{equation} \label{eq:sfo_expected_distance}
\mathbb{E}_{\zeta_t}[\|\theta_{t+1} - \theta\|^2] \le \|\theta_t - \theta\|^2 - 2\eta \langle F(y_t), \theta_t - \theta \rangle + \eta^2 \|F(y_t)\|^2 + \eta^2 \sigma^2.
\end{equation}
We decouple the inner product via $y_t$
\begin{equation} \label{eq:sfo_inner_split}
-2\eta \langle F(y_t), \theta_t - \theta \rangle = -2\eta \langle F(y_t), y_t - \theta \rangle - 2\eta \langle F(y_t), \theta_t - y_t \rangle.
\end{equation}
Substituting $\theta_t - y_t = 2\eta u_t \hat{F}_t$ and applying the polarization identity $-2\langle a, b \rangle = \|a-b\|^2 - \|a\|^2 - \|b\|^2$ yields
\begin{equation}
-4\eta^2 u_t \langle F(y_t), \hat{F}_t \rangle = 2\eta^2 u_t \|F(y_t) - \hat{F}_t\|^2 - 2\eta^2 u_t \|F(y_t)\|^2 - 2\eta^2 u_t \|\hat{F}_t\|^2.
\end{equation}
Merging this with $\eta^2 \|F(y_t)\|^2$ from \eqref{eq:sfo_expected_distance} yields $\eta^2 (1 - 2u_t) \|F(y_t)\|^2$. Define $A_t = F(y_t) - F_t$. Expanding $\|F(y_t)\|^2 = \|F_t\|^2 + 2\langle F_t, A_t \rangle + \|A_t\|^2$ yields
\begin{equation} \label{eq:sfo_f_y_expansion}
\eta^2 (1 - 2u_t) \|F_t\|^2 + 2\eta^2 (1 - 2u_t) \langle F_t, A_t \rangle + \eta^2 (1 - 2u_t) \|A_t\|^2.
\end{equation}
By Young's inequality
\begin{equation}
2\eta^2 (1 - 2u_t) \langle F_t, A_t \rangle \le \frac{\eta^2}{2} \|F_t\|^2 + 2\eta^2 (1 - 2u_t)^2 \|A_t\|^2.
\end{equation}
Defining the base oracle noise $N_t = F_t - \hat{F}_t$, we bound the variance term
\begin{equation}
2\eta^2 u_t \|F(y_t) - \hat{F}_t\|^2 = 2\eta^2 u_t \|A_t + N_t\|^2 \le 4\eta^2 u_t \|A_t\|^2 + 4\eta^2 u_t \|N_t\|^2.
\end{equation}
Aggregating the coefficients of $\|A_t\|^2$ yields $\eta^2 [ (1 - 2u_t) + 2(1 - 2u_t)^2 + 4u_t ] = \eta^2 (3 - 6u_t + 8u_t^2) := \eta^2 P(u_t)$. By $L$-Lipschitz continuity of $F$
\begin{equation}
\|A_t\|^2 \le L^2 \|y_t - \theta_t\|^2 = 4\eta^2 L^2 u_t^2 \|\hat{F}_t\|^2.
\end{equation}
Thus, the aggregate penalty is bounded by $4\eta^4 L^2 u_t^2 P(u_t) \|\hat{F}_t\|^2$. Let $Q(u_t) = 4 u_t^2 P(u_t)$. Integrating over the uniform distribution yields $\mathbb{E}[Q(u_t)] = 4.4 := C_R$. Taking the conditional expectation $\mathbb{E}_t[\cdot] = \mathbb{E}[\cdot \mid \mathcal{F}_t]$, independence of $u_t$ and $\xi_t$ yields
\begin{equation}
\begin{aligned}
\mathbb{E}_t \left[ \eta^2\left( \frac{3}{2} - 2u_t \right) \|F_t\|^2 \right] &= \frac{1}{2}\eta^2 \|F_t\|^2, \\
\mathbb{E}_t \left[ -2\eta^2 u_t \|\hat{F}_t\|^2 \right] &= -\eta^2 \big( \|F_t\|^2 + \mathbb{E}_t[\|N_t\|^2] \big), \\
\mathbb{E}_t \left[ 4\eta^2 u_t \|N_t\|^2 \right] &= 2\eta^2 \mathbb{E}_t[\|N_t\|^2], \\
\mathbb{E}_t \left[ \eta^4 L^2 Q(u_t) \|\hat{F}_t\|^2 \right] &= \eta^4 L^2 C_R \big( \|F_t\|^2 + \mathbb{E}_t[\|N_t\|^2] \big).
\end{aligned}
\end{equation}
Let $\sigma_t^2 = \mathbb{E}_t[\|N_t\|^2] \le \sigma^2$. Summing the expectations yields the coefficient for $\|F_t\|^2$ as $-\eta^2 \big( \frac{1}{2} - \eta^2 L^2 C_R \big)$. For $\eta \le 1 / (L \sqrt{2 C_R})$, this term is non-positive and can be discarded. The remaining variance terms, including $\eta^2 \sigma^2$ from \eqref{eq:sfo_expected_distance}, sum to
\begin{equation}
\eta^2 \sigma^2 - \eta^2 \sigma_t^2 + 2\eta^2 \sigma_t^2 + \eta^4 L^2 C_R \sigma_t^2 \le \eta^2 \big( 2 + \eta^2 L^2 C_R \big) \sigma^2.
\end{equation}
Substituting this variance bound into \eqref{eq:sfo_expected_distance} and using convax-concave structure yields
\begin{equation}
2\eta \mathbb{E} \left[ f(x_{y,t}, z) - f(x, z_{y,t}) \right] \le \mathbb{E} \left[ \|\theta_t - \theta\|^2 \right] - \mathbb{E} \left[ \|\theta_{t+1} - \theta\|^2 \right] + \eta^2 \big( 2 + \eta^2 L^2 C_R \big) \sigma^2.
\end{equation}
Summing from $t=0$ to $k$ telescopes
\begin{equation}
\sum_{t=0}^k \mathbb{E} \left[ f(x_{y,t}, z) - f(x, z_{y,t}) \right] \le \frac{\|\theta_0 - \theta\|^2}{2\eta} + \frac{(k+1)\eta \big( 2 + \eta^2 L^2 C_R \big)}{2} \sigma^2.
\end{equation}
Dividing by $k+1$ and applying Jensen's inequality to the jointly convex function $f(\cdot, z) - f(x, \cdot)$ evaluated at $\bar{y}_k$ yields
\begin{equation}
\mathbb{E} \left[ f(\bar{x}_k, z) - f(x, \bar{z}_k) \right] \le \frac{1}{k+1} \sum_{t=0}^k \mathbb{E} \left[ f(x_{y,t}, z) - f(x, z_{y,t}) \right].
\end{equation}
This completes the proof.
\end{proof}
\section{Numerical Verification}\label{sec:exp}
To empirically validate the advantages of the proposed randomized integration schemes, we evaluate \eqref{eq:eg} and \eqref{eq:rampage+} across three nonlinear examples. For each setting, we execute $1000$ independent Monte Carlo trials. 

We employ the same initialization for both methods within each experiment. We also select the step size such that it lies precisely at the instability threshold where \eqref{eq:eg} diverges due to its deterministic $\mathcal{O}(\eta^2)$ discretization bias. This highlights that the proposed unbiased integration and the variance-reduced architecture of \eqref{eq:rampage+} leads to convergence in aggressive step-size regimes where standard extrapolation of \eqref{eq:eg} fails. Beside running \eqref{eq:eg} with this unstable stepsize, we run another instance with a slightly lower stepsize which is convergent. In all cases, we also find the largest step size for \eqref{eq:rampage+} that still leads to convergence to verify our main claim that \eqref{eq:rampage+} has lower estimation error than \eqref{eq:eg} and as such is more stable.

\subsection{Conservative Polynomial Field}\label{sec:exp-1}
We first consider an unconstrained conservative vector field derived from a 4th-order polynomial objective, evaluated component-wise as
\begin{equation}
F_i(\theta) = \theta_i + 5\theta_i^3 - 6\theta_i^2, \quad\theta \in \mathbb{R}^{10}.
\end{equation}
As the state vector moves away from the origin, the operator's curvature grows rapidly. The deterministic truncation error of \eqref{eq:eg} is heavily penalized by these massive higher-order derivatives, systematically inducing trajectory overshoot and divergence. Conversely,  \eqref{eq:rampage+} bypasses this bottleneck via unbiased integration while enjoying a significantly low variance.

\subsection{High-Frequency Rotational Min-Max Game}\label{sec:exp-2}
Next, motivated by our infinite-GAN example discussed in Section \ref{sec:motivation} and detailed in Appendix \ref{sec:gan}, we evaluate a non-conservative, highly rotational min-max game with high-frequency terms. The driving operator is 
\begin{equation}
F(\theta) = M\theta + 0.005\omega \odot \sin(\omega \odot \theta), \quad\theta \in \mathbb{R}^{20}.
\end{equation}
The matrix $M \in \mathbb{R}^{20 \times 20}$ is block-diagonal, composed of $2 \times 2$ sub-blocks taking the form
\begin{equation}
B_i = \begin{bmatrix} 0.1 & \beta_i \\ -\beta_i & 0.1 \end{bmatrix}.
\end{equation}
The added trigonometric perturbation guarantees \eqref{eq:eg} fails due to deterministic truncation error as we discussed previously. \eqref{eq:rampage+}, on the other hand, maintains convergence. We uniformly space $\beta_i \in [2.0, 8.0]$ and $\omega_i \in [15.0, 45.0]$.

\subsection{2d High-Frequency Rotational Min-Max Game}\label{sec:exp-3}
We evaluate a 2-dimensional min-max game, again motivated by our infinite-GAN example. The driving operator is formulated as
\begin{equation}
F(\theta) = M\theta + 0.04 \odot \sin(\omega \odot \theta), \quad\theta \in \mathbb{R}^{2}.
\end{equation}
The matrix $M \in \mathbb{R}^{2 \times 2}$ takes the form
\begin{equation}
M = \begin{bmatrix} 0 & -1 \\ 1 & 0 \end{bmatrix}.
\end{equation}
The added trigonometric perturbation guarantees \eqref{eq:eg} fails due to deterministic truncation error as we discussed previously. \eqref{eq:rampage+}, on the other hand, maintains convergence. We set $\omega =25$.

\subsection{Distributionally Robust Optimization (DRO)}\label{sec:exp-4}
To evaluate the behavior of the proposed \eqref{eq:rampage+} in a non-bilinear convex-concave min-max problem, we consider a KL-regularized distributionally robust optimization (DRO) formulation for a binary logistic regression task. Specifically, given a dataset of $N$ samples, we solve the following saddle-point problem 

\begin{equation}
    \min_{\theta\in \mathbb{R}^d}\max_{v\in \mathbb{R}^N} \Phi(\theta, v) = \sum_{i=1}^Np_i(v)\ell_i(\theta) - \gamma \sum_{i=1}^Np_i(v) + \frac{\lambda}{2}\|\theta\|_2^2 - \frac{\alpha}{2}\|v\|_2^2,
\end{equation}where the adversarial distribution over samples is $p_i(v) = e^{v_i}/\sum_{j=1}^Ne^{v_j}$, and the inner maximization corresponds to a KL-divergence uncertainty set. The objective is a combination of weighted empirical logistic loss with entropy regularization on the adversarial distribution and quadratic regularization on both primal and dual variables. 

We formulate this problem on the Breast Cancer Wisconsin dataset~\cite{breast_cancer_wisconsin_diagnostic_17} with dimensionality $d=30$ and sample size $N=569$, and treat it as a binary classification task. Following standard preprocessing, all features are standardized  to zero mean and unit variance. Binay labels are mapepd to $\{-1, 1\}$, and the per-sample loss is chosen as the logistic loss 
\begin{equation}
    \ell_i (\theta) = \log(1 + \mathrm{exp}(-y_ix_i^T\theta)).
\end{equation}The operator $F(\theta, v)$ associated with the saddle-point problem is constructed from the gradient of the primal objective and the negative gradient of the dual objective. For numerical stability, logits are clipped before evaluating exponential terms. The entropy regularization parameter is set to $\gamma=0.1$, while the primal and dual quadratic regularization parameters are set to $\lambda =0.01$ and $\alpha=0.01$, respectively. 

We compare the standard~\eqref{eq:eg} method and the proposed~\eqref{eq:rampage+} method under both conservative and aggressive step-size regimes. For extragradient, we use step sizes $\eta=1.09$ and $\eta=1.10$, while for~\eqref{eq:rampage+} we consider $\eta=1.09$ and a substantially larger step size $\eta=2.0$. Each method is run for $500$ iterations over $100$ independent trials. For each trial, all methods are initialized from the same random primal variable, $\theta_0\sim\mathcal{N}(0,0.01^2I)$, and dual variable, $v_0=0$, while the initialization is resampled across trials. We report the mean operator residual norm $\|F(\theta_k,v_k)\|_2$ across trials, with shaded regions indicating one standard deviation around the mean.

The main objective of this experiment is to assess the robustness of~\eqref{eq:rampage+} to larger step sizes and compare its empirical convergence behavior against~\eqref{eq:eg} on the KL-regularized DRO problem. The results show that both methods converge stably under their conservative step-size regimes. However, a slight increase in the extragradient step size causes the residual norm to plateau at a larger step size, indicating instability. In contrast,~\eqref{eq:rampage+} remains stable and continues to reduce the residual norm even under a significantly larger step size, demonstrating improved robustness to aggressive step sizes.

\subsection{Adversarial Training}\label{sec:exp-5}
We further evaluate \eqref{eq:rampage+} on a non-bilinear min-max adversarial training problem for binary logistic regression. Given a  dataset of $N$ training samples, we solve the following saddle-point problem:
\begin{equation}
    \min_{\theta\in \mathbb{R}^d}\max_{\Delta\in \mathbb{R}^{N\times d}} \Phi(\theta, \Delta) = \sum_{i=1}^N\log\left(1 + \mathrm{exp}(-y_i \theta^T(x_i + \delta_i))\right) - \frac{\gamma}{2N}\sum_{i=1}^N\|\delta_i\|_2^2,
\end{equation}where $\theta$ denotes the model parameters, $\Delta=[\delta_1,\ldots,\delta_N]^T$ represents the adversarial perturbations, and $\gamma >0$ controls the strength of the quadratic regularization on the perturbations. We set $\gamma=1.0$.  

We use the same Breast Cancer Wisconsin dataset~\cite{breast_cancer_wisconsin_diagnostic_17} as in the previous experiments. We followed the same pre-processing steps, except that the logistic loss was evaluated on the adversarially perturbed inputs $x_i+\delta_i$. 

Each method is run for $1000$ iterations over $100$ independent trials. All methods are initialized from the same zero initialization in each trial. We compare the standard \eqref{eq:eg} and the proposed~\eqref{eq:rampage+} under both conservative and aggressive step size regimes. For~\eqref{eq:eg}, we use step size $\eta=1.41$ and $\eta =1.42$, while for~\eqref{eq:rampage+} we consider $\eta = 1.42$ and a substantially larger step size $\eta=2.23$. 

The main objective of this experiment is to assess the robustness of~\eqref{eq:rampage+} to larger step sizes and compare its convergence behavior against~\eqref{eq:eg} on the adversarial training problem. The results show that both methods converge stably under their conservative step size regimes. However, a slight increase in the~\eqref{eq:eg} causes the residual norm to plateau at a relatively large value, indicating instability. In contrast,~\eqref{eq:rampage+} remains stable and continues to reduce the residual norm under a significantly larger step size, demonstrating improved robustness to aggressive step sizes. 

\end{document}